\RequirePackage{fix-cm}
\documentclass{article}
\PassOptionsToPackage{table}{xcolor}
\usepackage{iclr2027_conference,times}
\usepackage[utf8]{inputenc}
\usepackage[T1]{fontenc}
\usepackage{url,booktabs,amsfonts,amsmath,amssymb,graphicx,array,multirow,placeins,algorithm,algpseudocode,capt-of,longtable}
\usepackage[protrusion=false]{microtype}
\usepackage[table]{xcolor}
\usepackage[breaklinks=true]{hyperref}
\hypersetup{hidelinks,pdfauthor={Xu Wang; Xunkai Li; Yinlin Zhu; Rong-Hua Li; Guoren Wang},pdftitle={GOMA: Toward Structure-Driven Multimodal Alignment from a Graph Signal Smoothing Perspective}}
\newcommand{\tabstat}[2]{#1{\fontsize{7}{8.4}\selectfont\,$\pm$\,#2}}

\newcommand{\tabbeststat}[2]{\textbf{#1}{\fontsize{7}{8.4}\selectfont\,$\pm$\,#2}}

\newcommand{\tabsecondstat}[2]{\underline{#1}{\fontsize{7}{8.4}\selectfont\,$\pm$\,#2}}

\newcommand{\paperTableSetup}{\fontsize{8.5}{10.2}\selectfont\setlength{\tabcolsep}{1.5pt}\renewcommand{\arraystretch}{1.10}}
\title{GOMA: Toward Structure-Driven Multimodal Alignment from a Graph Signal Smoothing Perspective}
\author{\normalfont
Xu Wang$^{1}$ \quad Xunkai Li$^{2}$ \quad Yinlin Zhu$^{3}$ \quad Rong-Hua Li$^{2}$ \quad Guoren Wang$^{2}$\\
$^{1}$School of Airspace Science and Engineering, Shandong University, WeiHai, China\\
$^{2}$Department of Computer Science, Beijing Institute of Technology, Beijing, China\\
$^{3}$School of Computer Science and Engineering, Sun Yat-sen University, GuangZhou, China\\
\textbf{Correspondence to:} Rong-Hua Li \texttt{<lironghuabit@126.com>}
}

\newcommand{\completeTableSetup}{\fontsize{8.5}{10.2}\selectfont\setlength{\tabcolsep}{2.3pt}\renewcommand{\arraystretch}{1.02}}

\newcommand{\ablationTableSetup}{\fontsize{8}{9.6}\selectfont\setlength{\tabcolsep}{0.8pt}\renewcommand{\arraystretch}{1.08}}
\newcommand{\ReadoutIdentityMin}{0.53}
\newcommand{\ReadoutIdentityMax}{2.04}
\newcommand{\ReadoutSemanticMin}{3.05}
\newcommand{\ReadoutSemanticMax}{7.49}

\newcommand{\GraphLossSemanticMin}{2.13}
\newcommand{\GraphLossSemanticMax}{6.80}

\begin{document}
\maketitle
\suppressfloats[t]
\begin{abstract}
Multimodal retrieval uses images, text, and object relationships to answer different questions about the same collection. A query may seek an object's paired description, another object in the same category, or an object connected by an observed relationship. These goals rely on different notions of relevance. Paired matching requires object-specific distinctions, whereas cross-object retrieval benefits from relational agreement. Existing methods learn strong cross-modal correspondence or propagate information over a graph, but a shared output regularized toward neighbors can weaken identity distinctions. Moreover, gains from graph regularization can diminish after uniform graph propagation. We introduce \textbf{G}raph-\textbf{O}ptimized \textbf{M}ultimodal \textbf{A}lignment (GOMA), which assigns these roles to two connected embeddings. Each modality produces a content embedding directly supervised for paired identity and a semantic embedding jointly trained with cross-modal pairs and observed relationships. For complete records, both embeddings form an initial fused representation, semantic agreement sets positive weights on observed edges, and restart graph propagation reinjects this initial signal. This design lets a jointly trained model support single-modality and dual-attribute retrieval through a task-specific readout. Across six datasets and four tasks, GOMA achieves state-of-the-art performance on all 14 primary measures against 14 external methods. Controlled comparisons further show how separate supervision, graph regularization, and semantic-guided graph propagation shape the final representation and align the learned signal with each retrieval target.
\end{abstract}

\section{Introduction}
An image can initiate different searches. A user may want its paired description, another product with the same purpose, or an object linked through previous interactions. Images and text describe object attributes, while relationships such as co-purchases or post interactions organize the collection \citep{clip,SigLIP22025,UniIR2024,VLM2Vec2025,MAGB,OpenMAG}. Because each search defines a different correct answer, multimodal alignment must expose the information required by its retrieval target.

Fig.~\ref{fig:background} illustrates these goals with two backpacks. Finding the blue backpack's own description requires distinguishing it from the red one, whereas searching for other backpacks makes their shared category relevant. Co-purchases can connect both similar and complementary products. Pulling related objects together may improve cross-object retrieval but make a true image-text pair harder to distinguish from similar alternatives. The challenge is to use relational agreement without losing the object-specific information required for paired matching.

\begin{figure}[!t]
\centering
\includegraphics[width=\linewidth]{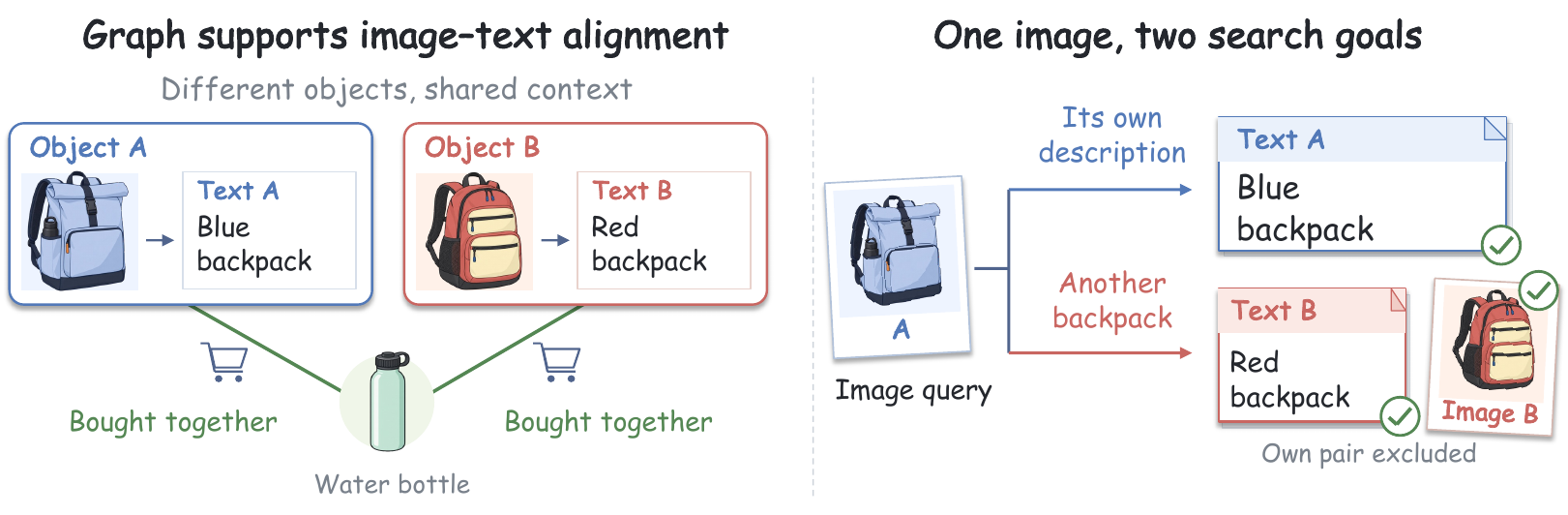}
\caption{Different retrieval goals within one collection. Identity matches an object's own attributes, semantic retrieval seeks other objects in the same category, and observed relations connect distinct objects.}
\label{fig:background}
\end{figure}

Paired ranking and contrastive objectives learn image-text correspondence, while generalized retrieval objectives extend supervision to richer definitions of relevance \citep{faghri2017vse++,clip,SigLIP,ClipRefine2025,GCLRank2024,GCLUniversal2025}. Multimodal graph models add neighborhood aggregation, filtering, and modality-topology interaction \citep{graph4mm,LGMRec,DGF,CAMPA2026,CoMAG2026,RoleMAG2026}. These methods support correspondence or relational information sharing, but they do not determine where graph supervision should act when one set of learned weights must serve both retrieval goals.

Graph signal processing provides a precise view of this tension. Graph smoothness measures how representations vary across observed relationships \citep{Shuman2013}. Relational agreement favors low variation, while paired identity still requires differences between objects. Intermediate and projected representations can also retain different information \citep{SimCLR2020,SMGFM2026}. The question is which representation graph smoothness should regulate and how its learned signal should guide later propagation.

Our empirical study verifies category information beyond degree-preserving graph relabelings, identifies an identity-semantic tradeoff in a shared output, and shows that uniform graph propagation reduces part of the gain supplied by graph regularization. These observations suggest two connected design principles. Identity and relational semantics should remain accessible through separately supervised embeddings, and the learned semantic signal should determine how information is shared during graph propagation.

To this end, we propose GOMA. Each modality produces a content embedding for paired identity and a semantic embedding learned from cross-modal pairs and observed relationships. Single-modality retrieval uses the embedding that matches its target. For complete records, both embeddings form the initial fused representation $\mathbf H_0$, semantic agreement weights observed edges, and restart graph propagation reinjects $\mathbf H_0$ at every step. The semantic embedding therefore connects graph regularization to propagation by controlling the exchange between neighbors.

\noindent\textbf{Our Contributions.}
(1) \underline{\textit{Valuable Insights.}} We show that retrieval relevance determines where graph smoothness should act and why graph regularization and uniform graph propagation need not provide additive gains.
(2) \underline{\textit{Novel Method.}} We propose GOMA, which jointly learns separately supervised content and semantic embeddings, uses semantic agreement to weight observed edges, and applies restart graph propagation to complete records. One set of learned weights supports four retrieval tasks through a task-specific readout.
(3) \underline{\textit{SOTA Performance.}} Across six datasets and four retrieval tasks, GOMA leads all 14 primary measures against 14 external methods. Core removals, equal-capacity comparisons, and signal measurements identify the roles of its components.

\section{Problem Formulation}
\label{sec:problem}
We first define the data and observed graph, followed by the retrieval interfaces and relevance rules. This separates model inputs from evaluation targets.

\subsection{Data and Graph Setup}
Let $\mathcal V$ contain $N=|\mathcal V|$ objects and let $\mathcal M=\{v,t\}$ denote the image and text modalities. Object $i\in\mathcal V$ has paired attributes represented by prepared features $\mathbf x_i^m\in\mathbb R^d$ for each $m\in\mathcal M$. An observed graph $\mathcal G=(\mathcal V,\mathcal E)$ records relationships between distinct objects, such as co-purchases or interactions. The paired attributes and graph are model inputs. Category labels and held-out relationships define analysis or evaluation targets.

\subsection{Retrieval Interfaces}
The inference input determines which information may be used for scoring. In \emph{single-modality retrieval}, an image query is compared with text candidates, or a text query with image candidates. Every object is encoded from its own supplied attribute, without its missing attribute or an evaluation-time neighborhood. In \emph{dual-attribute retrieval}, both attributes and the permitted observed graph are available for graph propagation and offline index construction. The interface defines the available information, while the rules below define a correct candidate.

\subsection{Retrieval Tasks}
For a query object $i$, the system ranks candidate objects $j$ from the prescribed candidate set. The four tasks combine the two interfaces with three definitions of relevance.

\textbf{Single-modality identity and semantic retrieval.}
Identity retrieval asks for the opposite-modality attribute of the same object. Candidate $j$ is relevant to query $i$ exactly when $j=i$. We report recall at ranks 1, 5, and 10 in both directions, with bidirectional R@1 as the primary measure. Cross-object semantic retrieval instead seeks another object in the same category. Let $y_i$ denote the evaluation category of object $i$. A candidate is relevant when $y_j=y_i$ and $j\ne i$. Query and gallery objects are disjoint, so the query's paired attribute cannot satisfy this task. The primary measure is category-macro mAP@10, which first averages query scores within each category and then gives all categories equal weight.

\textbf{Dual-attribute semantic retrieval and relation ranking.}
Semantic retrieval keeps the same cross-object relevance rule while allowing both attributes and the permitted observed graph. Its primary measure remains category-macro mAP@10. Relation ranking asks the model to recover a missing relationship by ranking the true linked object among prescribed negative candidates. Its primary measure is mean reciprocal rank (MRR). Before propagation, each target relationship and its reverse orientation are removed from the supplied graph. Its endpoints may retain other observed neighbors, so the task measures whether the remaining context helps recover the missing link.

\subsection{Training and Evaluation Setting}
Let $\mathcal T\subset\mathcal V$ be the training objects and let $\mathcal E_{\mathcal T}$ contain the observed edges whose endpoints both lie in $\mathcal T$. Learning uses their paired attributes and the induced graph $(\mathcal T,\mathcal E_{\mathcal T})$. Category labels are used only for analysis and evaluation and never supervise GOMA. Hidden relation targets are removed from the observed graph before propagation. The goal is one set of learned parameters per dataset and seed that supports every applicable task. Sec.~\ref{sec:experiments} specifies model selection, and Appendix~\ref{app:protocol} gives the partitions, candidate pools, and metric averaging.

\section{Empirical Motivation}
\label{sec:empirical}
We examine relational information, retrieval targets, and the interaction between graph regularization and propagation. Graph analysis uses training objects and categories, representation analysis uses validation candidates, and the propagation probe uses matched training steps on the established evaluation pools. Appendix~\ref{app:empirical} gives complete definitions and measurements.

\begin{figure}[!t]
\centering
\includegraphics[width=\linewidth]{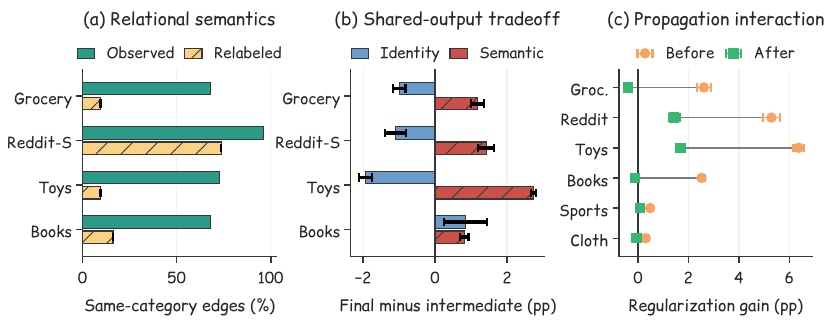}
\caption{Graph semantics, shared-output retrieval, and propagation. Error bars give graph-relabeling SD in (a) and paired model-seed SD in (b,c).}
\label{fig:empirical}
\end{figure}

\textbf{Relationships contain semantic information.}
Fig.~\ref{fig:empirical}(a) compares same-category training-edge proportions with three vertex relabelings within degree groups. Observed agreement is higher on all four category datasets by 22.67 to 62.97 percentage points. Because the control retains each object's training degree, the difference shows category structure beyond that explained by degree and category frequencies. The graph can therefore supply cross-object information that is useful to investigate for semantic retrieval.

\textbf{Different targets favor different embeddings.}
The shared-output probe has two residual mappings per modality, with identity, paired, and graph objectives at the final output. At the same weights, the final output improves semantic mAP@10 on all four datasets but reduces identity R@1 on three (Fig.~\ref{fig:empirical}(b)). Books improves in both. The Books exception indicates that the shared-output trade-off is dataset-dependent rather than universal. Graph regularization still updates the content mapping through the semantic branch, so this comparison concerns supervision placement. We therefore keep identity supervision at the content embedding and use the semantic embedding for cross-object retrieval.

\textbf{Graph regularization and propagation can interact negatively.}
The last probe crosses graph regularization and uniform graph propagation at the same training step and with the same initialization rule. Graph regularization improves the pre-propagation retrieval score on all six datasets, but uniform propagation reduces part of that gain in every seed (Fig.~\ref{fig:empirical}(c)). Because propagation still helps both training variants, the result identifies overlapping gains rather than universal harm. It motivates using the learned semantics to guide propagation.

\section{Methodology}
\label{sec:method}
Fig.~\ref{fig:framework} organizes GOMA into joint representation learning and semantic-guided graph propagation. Each modality produces content and semantic embeddings. The learned semantics guide propagation for complete records, connecting one model to all four retrieval tasks.

\subsection{Content and Semantic Embeddings}
Let $\mathcal N$ normalize each row to unit length and let $m\in\{v,t\}$ denote the modality. From normalized input $\mathbf x_i^m$, GOMA computes
\begin{equation}
\mathbf c_i^m=\mathcal N\!\left(\mathbf x_i^m+f_m(\mathbf x_i^m)\right),
\qquad
\mathbf z_i^m=\mathcal N\!\left(\mathbf c_i^m+g_m(\mathbf c_i^m)\right).
\label{eq:readouts}
\end{equation}
The two-layer mappings $f_m$ and $g_m$ have width $d$ and a GELU activation. Their final linear layers start at zero, initially retaining the supplied content. The content embedding $\mathbf c_i^m$ serves object identity, while $\mathbf z_i^m$ expresses relationships between objects. Each computation uses only its own modality.

The content embedding receives bidirectional VSE++ hard-negative ranking loss $\mathcal L_{\mathrm{id}}$ with margin 0.2, separating each true pair from its most competitive incorrect candidate \citep{faghri2017vse++}. At the semantic embedding, bidirectional paired cross-entropy $\mathcal L_{\mathrm{pair}}$ with temperature 0.1 maintains image-text correspondence. Both objectives update the content mapping. Identity remains directly supervised while the jointly learned semantic transformation expresses cross-object agreement.

\subsection{Graph-Regularized Training}
Graph smoothness encourages the semantic outputs to agree across training relationships. Let $d_i$ be the degree of training object $i$ in the complete induced training graph. We minimize the degree-normalized edge energy
\begin{equation}
\mathcal E_{\mathcal G}(\mathbf Z^v,\mathbf Z^t)=
\frac{1}{2d|\mathcal E_{\mathcal T}|}
\sum_{m\in\{v,t\}}\sum_{(i,j)\in\mathcal E_{\mathcal T}}
\left\|\frac{\mathbf z_i^m}{\sqrt{d_i}}-
\frac{\mathbf z_j^m}{\sqrt{d_j}}\right\|_2^2.
\label{eq:energy}
\end{equation}
Here $\mathbf Z^m$ stacks the semantic embeddings and the edge list contains both orientations. Degree normalization accounts for the different numbers of relationships attached to objects. Isolated training objects contribute through paired supervision. This is the normalized graph Dirichlet principle applied within each modality \citep{Shuman2013}. Cross-modal pairing establishes correspondence, and graph energy transfers the observed relational structure into independently computable attribute representations.

The combined objective is
\begin{equation}
\mathcal L=\mathcal L_{\mathrm{id}}+
\eta\left(\mathcal L_{\mathrm{pair}}+\lambda\mathcal E_{\mathcal G}\right).
\label{eq:objective}
\end{equation}

\begin{figure}[!t]
\centering
\includegraphics[width=\linewidth]{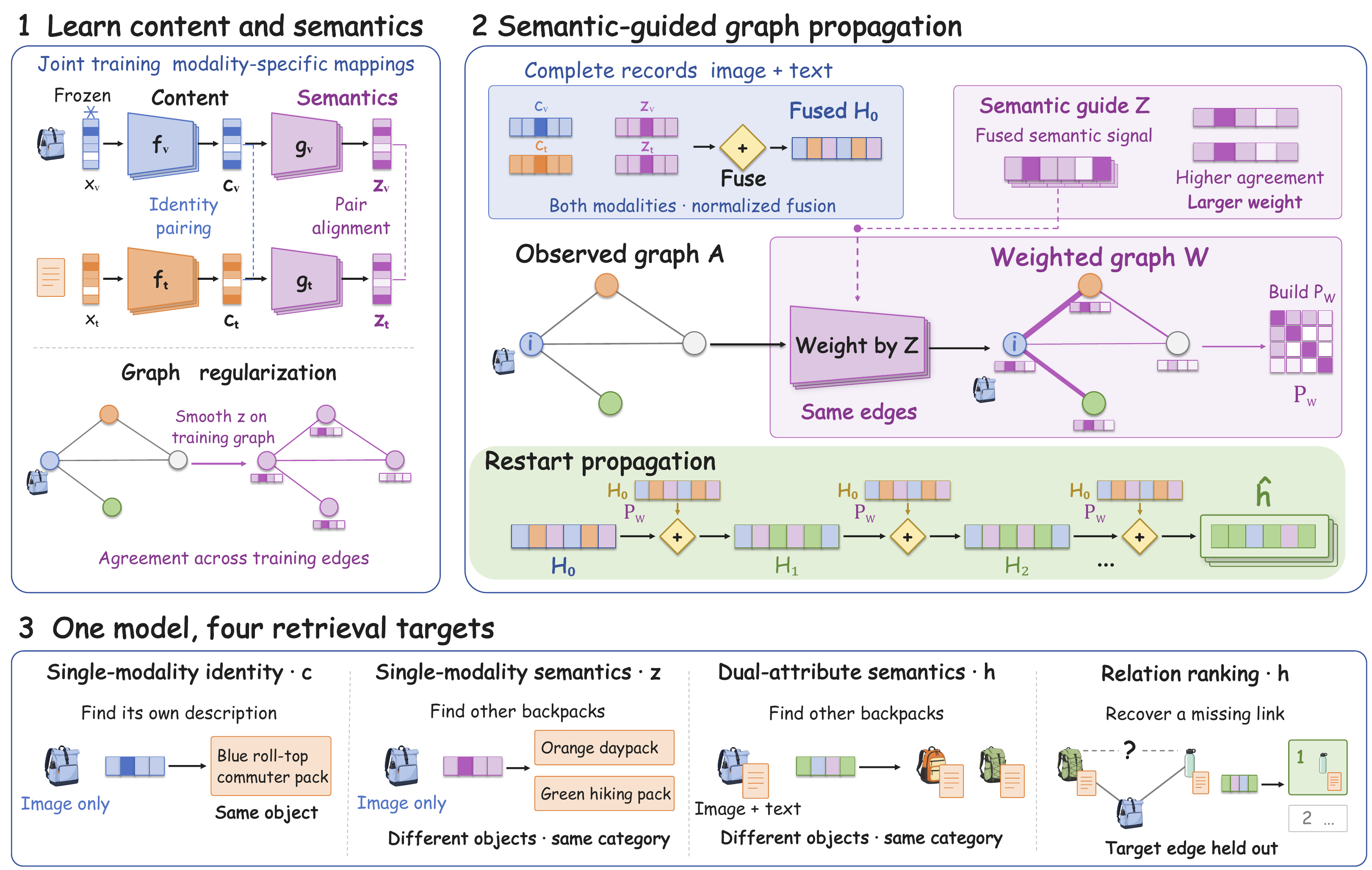}
\caption{GOMA jointly learns content and semantic embeddings. For complete records, both form the initial fused representation and the learned semantics guide graph propagation. One model supports all four retrieval tasks.}
\label{fig:framework}
\end{figure}
We set $\eta$ and $\lambda$ once from gradient norms on ten training batches at initialization, then retain them throughout learning. Each update samples paired objects and training edges independently. Appendix~\ref{app:optimization} gives the exact objectives and loss-weight construction. The observed graph influences single-modality retrieval through parameters learned from the training graph. At inference, each query encoder uses only its supplied attribute.

\subsection{Semantic-Guided Graph Propagation}
For complete records, we fuse the two attributes into content $\mathbf C=\mathcal N(\mathbf C^v+\mathbf C^t)$ and semantics $\mathbf Z=\mathcal N(\mathbf Z^v+\mathbf Z^t)$, where each matrix stacks the corresponding embedding. Their equal fusion $\mathbf H_0=\mathcal N(\mathbf C+\mathbf Z)$ defines the initial signal for graph propagation.

The semantic signal also describes which neighbors agree in the learned retrieval space. For each allowed observed edge, we set
\begin{equation}
W_{ij}=A_{ij}\frac{\sigma^2}{\sigma^2+\|\mathbf Z_i-\mathbf Z_j\|_2^2}.
\label{eq:conductance}
\end{equation}
Here $\mathbf A$ is the binary symmetric adjacency of the permitted observed graph. The scale $\sigma^2$ is the mean squared semantic distance over legal training edges, with a numerical floor of $10^{-8}$. Close semantic neighbors receive larger weights, while every observed edge retains positive weight. Similarity-based weighting is an established way to guide graph aggregation \citep{TWIRLS2021}. GOMA uses it to connect the learned semantic representation to graph propagation. The weights are computed once and remain unchanged during propagation.

In the example of Fig.~\ref{fig:background}, the graph can connect similar backpacks and complementary accessories. Learned semantic distances adjust their influence while retaining both relationships. Graph regularization thus affects propagation through both the initial signal and the edge weights.

Let $\mathbf D_W$ contain weighted degrees and $\mathbf P_W=\tfrac12(\mathbf I+\mathbf D_W^{-1/2}\mathbf W\mathbf D_W^{-1/2})$. Isolated objects use $(P_W)_{ii}=1$. GOMA propagates the initial signal through
\begin{equation}
\mathbf H_{k+1}=\alpha\mathbf H_0+(1-\alpha)\mathbf P_W\mathbf H_k,
\qquad k=0,\ldots,K-1.
\label{eq:propagation}
\end{equation}
We use $\alpha=0.1$ and $K=8$ for every dataset, followed by row normalization $\widehat{\mathbf H}=\mathcal N(\mathbf H_K)$. The restart term $\alpha\mathbf H_0$ reinjects the initial fused representation at every step \citep{APPNP2019}. For the computed operator, the recurrence is a finite low-pass graph filter that retains a positive contribution from the initial signal at every frequency. Final row normalization is a separate nonlinear operation. Appendix~\ref{app:theory} derives the filter and distinguishes these two steps.

\subsection{Retrieval Scoring}
The required relevance determines the score. For image-to-text and dual-attribute retrieval, GOMA uses
\begin{equation}
s(i,j)=
\begin{cases}
(\mathbf c_i^v)^\top\mathbf c_j^t, & \text{single-modality identity},\\
(\mathbf z_i^v)^\top\mathbf z_j^t, & \text{single-modality semantic},\\
\widehat{\mathbf h}_i^\top\widehat{\mathbf h}_j, & \text{dual-attribute semantic or relation}.
\end{cases}
\label{eq:scoring}
\end{equation}
Text-to-image retrieval exchanges $v$ and $t$. The first two scores encode attributes independently, while the third uses complete records and permitted relationships. All tasks share learned weights. Offline index construction costs $O(|\mathcal E|d)$ for edge weighting and $O(K(|\mathcal E|+N)d)$ for propagation. Subsequent ranking uses cached representations.

The retrieval interface determines the choice among $\mathbf c$, $\mathbf z$, and $\widehat{\mathbf h}$. This task-specific readout distinguishes identity, cross-object semantic, and relation relevance.

\section{Experiments}
\label{sec:experiments}
\subsection{Experimental Setup}
\textbf{Datasets and tasks.}
Grocery, Reddit-S, Toys, and Books support identity and cross-object semantic retrieval. Sports and Cloth support relation ranking. These multimodal graph collections contain 15,894 to 685,294 objects \citep{MAGB,OpenMAG}. All methods receive the same prepared features, object partitions, candidates, and relevance rules. The two semantic interfaces share query/gallery partitions and differ in their available attributes and graph information.

\textbf{External methods.}
We compare the same 14 methods in both interfaces. Vision-language methods are VSE++, CLIP-Refine, and SigLIP \citep{faghri2017vse++,ClipRefine2025,SigLIP}. Graph encoders are GCN, GAT, and GraphSAGE \citep{GCN,GAT,GraphSAGE}. Multimodal graph methods are MMGCN, MGAT, LGMRec, DMGC, DGF, CAMPA, CoMAG, and RoleMAG \citep{MMGCN,MGAT,LGMRec,DMGC,DGF,CAMPA2026,CoMAG2026,RoleMAG2026}. Their adaptations retain their core computations and losses on the common feature inputs. Graph methods use their graph or fusion outputs for dual-attribute retrieval and independently computable missing-input views for single-modality queries. Appendix~\ref{app:baselines} describes these interfaces.
VSE++, CLIP-Refine, and SigLIP are direct vision-language comparators for single-modality identity retrieval. GCN, GAT, GraphSAGE, and the multimodal graph methods are missing-modality adaptation controls that retain their graph structure and auxiliary objectives while using independently computable views.

\textbf{Selection and reporting.}
We report percentage means and sample SD over seeds 43, 44, and 45. GOMA uses AdamW with learning rate 0.001, weight decay 0.0001, and at most 500 updates, validating every five updates with 50-update patience. Category datasets select by the mean of identity R@1 and both semantic mAP@10 measures, while relation datasets select by MRR. One GOMA model per dataset and seed serves all tasks. External methods retain their task-specific selectors. Bold and underline mark the best and second-best displayed means. Appendix~\ref{app:complete} includes all metrics and directional recalls.

\subsection{Performance Comparison}

GOMA leads both primary single-modality measures on all four datasets in Table~\ref{tab:single}. Semantic mAP@10 gains over the strongest external method range from 3.33 to 8.85 points, while identity R@1 improves by 0.17 to 0.91 points. The same model supplies both outputs. These results show that training relationships can improve cross-object search while retaining a strong representation for paired identity retrieval.

\begin{table}[!htbp]
\caption{Single-modality identity R@1 and semantic mAP@10. Each query and candidate uses only its own supplied attribute.}
\label{tab:single}
\centering\paperTableSetup
\begin{tabular}{@{}l*{8}{c}@{}}
\toprule
\multirow{2}{*}{Method} & \multicolumn{2}{c}{\textbf{Grocery}} & \multicolumn{2}{c}{\textbf{Reddit-S}} & \multicolumn{2}{c}{\textbf{Toys}} & \multicolumn{2}{c}{\textbf{Books}} \\
\cmidrule(lr){2-3}\cmidrule(lr){4-5}\cmidrule(lr){6-7}\cmidrule(lr){8-9}
 & Id. R@1 & Sem. mAP & Id. R@1 & Sem. mAP & Id. R@1 & Sem. mAP & Id. R@1 & Sem. mAP \\
\midrule
\multicolumn{9}{l}{\textbf{Vision-language learning}} \\
VSE++ & \tabstat{49.94}{0.10} & \tabstat{30.53}{0.14} & \tabstat{16.10}{0.40} & \tabstat{37.34}{0.26} & \tabstat{59.19}{0.12} & \tabstat{23.09}{0.09} & \tabsecondstat{39.23}{0.09} & \tabstat{39.99}{0.29} \\
CLIP-Refine & \tabstat{49.93}{0.07} & \tabstat{33.72}{0.06} & \tabsecondstat{18.48}{0.11} & \tabstat{46.19}{0.10} & \tabsecondstat{62.57}{0.10} & \tabstat{27.57}{0.05} & \tabstat{35.44}{0.01} & \tabsecondstat{42.82}{0.02} \\
SigLIP & \tabsecondstat{52.36}{0.30} & \tabsecondstat{38.27}{0.12} & \tabstat{17.14}{0.19} & \tabsecondstat{50.89}{0.04} & \tabstat{61.29}{0.49} & \tabsecondstat{32.75}{0.25} & \tabstat{37.84}{0.07} & \tabstat{42.42}{0.07} \\
\midrule
\multicolumn{9}{l}{\textbf{Graph encoders}} \\
GCN & \tabstat{4.60}{0.13} & \tabstat{8.86}{0.53} & \tabstat{0.97}{0.24} & \tabstat{11.16}{1.20} & \tabstat{2.87}{0.31} & \tabstat{4.55}{0.06} & \tabstat{0.32}{0.04} & \tabstat{16.65}{0.99} \\
GAT & \tabstat{5.97}{1.07} & \tabstat{11.96}{0.42} & \tabstat{1.20}{0.17} & \tabstat{16.41}{1.46} & \tabstat{4.64}{0.78} & \tabstat{5.08}{0.49} & \tabstat{0.91}{0.19} & \tabstat{23.08}{0.67} \\
GraphSAGE & \tabstat{1.82}{0.38} & \tabstat{6.78}{0.45} & \tabstat{0.56}{0.16} & \tabstat{10.17}{1.26} & \tabstat{1.17}{0.21} & \tabstat{3.46}{0.17} & \tabstat{0.38}{0.07} & \tabstat{17.29}{1.20} \\
\midrule
\multicolumn{9}{l}{\textbf{Multimodal graph learning}} \\
MMGCN & \tabstat{15.33}{1.13} & \tabstat{28.15}{0.53} & \tabstat{4.42}{0.24} & \tabstat{42.22}{0.36} & \tabstat{14.26}{0.97} & \tabstat{16.94}{0.76} & \tabstat{12.44}{0.32} & \tabstat{40.82}{0.73} \\
MGAT & \tabstat{20.82}{0.29} & \tabstat{30.54}{1.00} & \tabstat{5.92}{0.28} & \tabstat{43.60}{0.98} & \tabstat{22.81}{0.73} & \tabstat{20.51}{0.31} & \tabstat{17.23}{0.54} & \tabstat{41.71}{0.12} \\
LGMRec & \tabstat{9.39}{0.60} & \tabstat{12.88}{0.23} & \tabstat{2.16}{0.41} & \tabstat{18.79}{0.39} & \tabstat{8.61}{0.11} & \tabstat{6.50}{0.43} & \tabstat{1.54}{0.02} & \tabstat{25.63}{0.98} \\
DMGC & \tabstat{24.05}{0.95} & \tabstat{34.70}{0.23} & \tabstat{6.88}{0.31} & \tabstat{50.60}{0.27} & \tabstat{23.74}{0.68} & \tabstat{17.48}{0.35} & \tabstat{10.36}{0.47} & \tabstat{42.14}{0.37} \\
DGF & \tabstat{20.29}{0.62} & \tabstat{32.45}{0.96} & \tabstat{5.71}{0.48} & \tabstat{44.62}{0.97} & \tabstat{20.81}{0.41} & \tabstat{23.01}{0.64} & \tabstat{13.11}{0.52} & \tabstat{40.89}{0.27} \\
CAMPA & \tabstat{3.25}{1.15} & \tabstat{8.76}{0.78} & \tabstat{0.73}{0.14} & \tabstat{10.52}{1.77} & \tabstat{3.69}{0.75} & \tabstat{5.37}{0.52} & \tabstat{0.24}{0.06} & \tabstat{16.07}{1.49} \\
CoMAG & \tabstat{12.15}{0.87} & \tabstat{13.84}{0.75} & \tabstat{2.34}{0.34} & \tabstat{17.05}{1.51} & \tabstat{11.20}{0.91} & \tabstat{7.18}{0.22} & \tabstat{1.46}{0.16} & \tabstat{26.37}{0.61} \\
RoleMAG & \tabstat{5.57}{0.25} & \tabstat{9.32}{0.31} & \tabstat{0.98}{0.11} & \tabstat{13.06}{0.65} & \tabstat{4.07}{0.89} & \tabstat{4.93}{0.33} & \tabstat{0.61}{0.06} & \tabstat{20.76}{0.85} \\
\midrule
\multicolumn{9}{l}{\textbf{Proposed method}} \\
\textbf{GOMA} & \tabbeststat{52.96}{0.06} & \tabbeststat{42.82}{0.22} & \tabbeststat{18.66}{0.19} & \tabbeststat{59.74}{0.26} & \tabbeststat{62.94}{0.40} & \tabbeststat{41.52}{0.44} & \tabbeststat{40.14}{0.31} & \tabbeststat{46.16}{0.03} \\
\bottomrule
\end{tabular}
 \end{table}

\begin{table}[!htbp]
\caption{Dual-attribute retrieval across six datasets. The first four report semantic mAP@10 and the last two report relation MRR.}
\label{tab:index}
\centering\paperTableSetup
\begin{tabular*}{\linewidth}{@{\extracolsep{\fill}}l*{6}{c}@{}}
\toprule
\multirow{3}{*}{Method} & \multicolumn{4}{c}{\textbf{Cross-object semantic retrieval}} & \multicolumn{2}{c}{\textbf{Relation ranking}} \\
\cmidrule(lr){2-5}\cmidrule(lr){6-7}
 & \textbf{Grocery} & \textbf{Reddit-S} & \textbf{Toys} & \textbf{Books} & \textbf{Sports} & \textbf{Cloth} \\
 & mAP@10 & mAP@10 & mAP@10 & mAP@10 & MRR & MRR \\
\midrule
\multicolumn{7}{l}{\textbf{Vision-language learning}} \\
VSE++ & \tabstat{39.67}{0.13} & \tabstat{56.19}{0.39} & \tabstat{29.58}{0.64} & \tabstat{46.23}{0.43} & \tabstat{62.61}{0.77} & \tabstat{59.20}{0.87} \\
CLIP-Refine & \tabstat{34.81}{0.32} & \tabstat{59.37}{2.67} & \tabstat{25.79}{0.76} & \tabstat{49.98}{0.02} & \tabstat{67.15}{0.02} & \tabstat{62.29}{0.02} \\
SigLIP & \tabsecondstat{44.17}{0.08} & \tabstat{62.95}{0.11} & \tabsecondstat{38.44}{0.10} & \tabstat{48.84}{0.13} & \tabstat{67.09}{0.16} & \tabstat{61.36}{1.64} \\
\midrule
\multicolumn{7}{l}{\textbf{Graph encoders}} \\
GCN & \tabstat{36.50}{1.06} & \tabstat{64.61}{0.86} & \tabstat{36.15}{0.93} & \tabstat{62.90}{0.21} & \tabsecondstat{87.50}{0.18} & \tabstat{80.34}{0.11} \\
GAT & \tabstat{35.67}{0.64} & \tabstat{67.93}{0.67} & \tabstat{37.76}{0.99} & \tabstat{62.03}{0.23} & \tabstat{86.67}{0.07} & \tabstat{79.48}{0.12} \\
GraphSAGE & \tabstat{28.95}{1.89} & \tabstat{58.01}{1.22} & \tabstat{25.90}{1.64} & \tabstat{61.06}{0.40} & \tabstat{82.05}{0.42} & \tabstat{75.33}{0.03} \\
\midrule
\multicolumn{7}{l}{\textbf{Multimodal graph learning}} \\
MMGCN & \tabstat{35.21}{0.82} & \tabstat{64.74}{1.29} & \tabstat{26.84}{0.65} & \tabstat{61.14}{0.57} & \tabstat{75.48}{0.43} & \tabstat{67.29}{0.29} \\
MGAT & \tabstat{40.65}{0.75} & \tabstat{67.43}{1.69} & \tabstat{29.65}{0.04} & \tabstat{60.58}{0.43} & \tabstat{73.58}{0.41} & \tabstat{66.23}{0.45} \\
LGMRec & \tabstat{41.80}{1.64} & \tabsecondstat{69.29}{0.44} & \tabstat{36.89}{0.35} & \tabsecondstat{62.99}{0.40} & \tabstat{87.07}{0.11} & \tabsecondstat{80.54}{0.32} \\
DMGC & \tabstat{32.56}{11.45} & \tabstat{55.13}{1.47} & \tabstat{13.53}{0.58} & \tabstat{48.56}{0.23} & \tabstat{57.38}{16.31} & \tabstat{54.10}{13.77} \\
DGF & \tabstat{34.57}{0.54} & \tabstat{61.93}{0.14} & \tabstat{27.34}{3.98} & \tabstat{59.93}{0.32} & \tabstat{77.06}{0.37} & \tabstat{71.91}{0.85} \\
CAMPA & \tabstat{30.12}{1.15} & \tabstat{56.32}{0.59} & \tabstat{24.53}{1.37} & \tabstat{47.21}{1.07} & \tabstat{64.07}{0.34} & \tabstat{56.14}{2.12} \\
CoMAG & \tabstat{40.31}{1.91} & \tabstat{67.28}{1.53} & \tabstat{34.36}{2.69} & \tabstat{59.38}{1.28} & \tabstat{84.54}{0.64} & \tabstat{74.33}{2.55} \\
RoleMAG & \tabstat{38.11}{0.32} & \tabstat{65.43}{1.21} & \tabstat{36.35}{1.02} & \tabstat{62.26}{0.28} & \tabstat{85.20}{0.83} & \tabstat{76.83}{0.52} \\
\midrule
\multicolumn{7}{l}{\textbf{Proposed method}} \\
\textbf{GOMA} & \tabbeststat{49.60}{0.05} & \tabbeststat{78.97}{0.05} & \tabbeststat{51.31}{0.07} & \tabbeststat{66.70}{0.01} & \tabbeststat{87.89}{0.03} & \tabbeststat{81.32}{0.11} \\
\bottomrule
\end{tabular*}
 \end{table}

Table~\ref{tab:index} gives the six dual-attribute primary results. GOMA improves semantic mAP@10 by 3.71 to 12.86 points and relation MRR by 0.39 to 0.78 points over the strongest external methods. It also leads all 18 dual-attribute metrics and all 12 single-modality semantic metrics in the complete tables. Together, the two interfaces yield state-of-the-art results on all 14 primary measures and 41 of the 42 aggregate measures. The component studies below examine how the learned representations and observed graph produce these results.

\subsection{Ablation Study}
\label{sec:ablation}
\begin{table}[!htbp]
\caption{Core graph ablations across all six dual-attribute datasets.}
\label{tab:ablation}
\centering\paperTableSetup
\begin{tabular*}{\linewidth}{@{\extracolsep{\fill}}l*{6}{c}@{}}
\toprule
Variant & \textbf{Grocery} & \textbf{Reddit-S} & \textbf{Toys} & \textbf{Books} & \textbf{Sports} & \textbf{Cloth} \\
\midrule
\textbf{GOMA} & \tabbeststat{49.60}{0.05} & \tabbeststat{78.97}{0.05} & \tabbeststat{51.31}{0.07} & \tabbeststat{66.70}{0.01} & \tabbeststat{87.89}{0.03} & \tabsecondstat{81.32}{0.11} \\
\midrule
w/o graph regularization & \tabsecondstat{49.55}{0.09} & \tabsecondstat{77.24}{0.10} & \tabsecondstat{49.18}{0.07} & \tabsecondstat{66.62}{0.03} & \tabsecondstat{87.78}{0.02} & \tabbeststat{81.46}{0.00} \\
w/o graph propagation & \tabstat{47.46}{0.25} & \tabstat{69.81}{0.18} & \tabstat{45.65}{0.41} & \tabstat{51.01}{0.14} & \tabstat{68.24}{0.11} & \tabstat{63.36}{0.15} \\
\bottomrule
\end{tabular*}
 \end{table}

The w/o graph regularization variant removes $\mathcal E_{\mathcal G}$ while retaining both embeddings, paired objectives, and semantic-guided propagation. The w/o graph propagation variant scores $\mathbf H_0$ directly. Both are retrained under the same validation rule.

Both removals reduce the primary metric on five datasets. On Cloth, graph propagation remains essential, while graph regularization produces a small MRR decrease.

\begin{table}[!htbp]
\caption{Separate supervision and semantic mapping for single-modality retrieval.}
\label{tab:readout}
\centering\ablationTableSetup
\begin{tabular}{@{}l*{8}{c}@{}}
\toprule
\multirow{2}{*}{Variant} & \multicolumn{2}{c}{\textbf{Grocery}} & \multicolumn{2}{c}{\textbf{Reddit-S}} & \multicolumn{2}{c}{\textbf{Toys}} & \multicolumn{2}{c}{\textbf{Books}} \\
\cmidrule(lr){2-3}\cmidrule(lr){4-5}\cmidrule(lr){6-7}\cmidrule(lr){8-9}
 & Id. R@1 & Sem. mAP & Id. R@1 & Sem. mAP & Id. R@1 & Sem. mAP & Id. R@1 & Sem. mAP \\
\midrule
\textbf{GOMA} & \tabbeststat{52.96}{0.06} & \tabbeststat{42.82}{0.22} & \tabbeststat{18.66}{0.19} & \tabbeststat{59.74}{0.26} & \tabbeststat{62.94}{0.40} & \tabbeststat{41.52}{0.44} & \tabbeststat{40.14}{0.31} & \tabbeststat{46.16}{0.03} \\
\midrule
Shared readout & \tabstat{51.79}{0.13} & \tabsecondstat{39.68}{0.16} & \tabstat{18.13}{0.30} & \tabstat{52.26}{0.45} & \tabsecondstat{60.89}{0.23} & \tabstat{37.36}{0.60} & \tabstat{39.00}{0.10} & \tabsecondstat{43.11}{0.10} \\
w/o semantic mapping & \tabsecondstat{51.83}{0.12} & \tabstat{39.53}{0.29} & \tabsecondstat{18.21}{0.31} & \tabsecondstat{52.71}{0.47} & \tabstat{60.86}{0.17} & \tabsecondstat{37.43}{0.14} & \tabsecondstat{39.04}{0.07} & \tabstat{43.05}{0.14} \\
\bottomrule
\end{tabular}
 \end{table}

Table~\ref{tab:readout} separates supervision placement from capacity. Shared readout keeps all parameters but applies identity supervision and scoring at $\mathbf z$. GOMA gains \ReadoutIdentityMin{} to \ReadoutIdentityMax{} points in identity R@1 and \ReadoutSemanticMin{} to \ReadoutSemanticMax{} in semantic mAP@10. Removing the semantic mapping ($\mathbf z=\mathbf c$) also lowers both primary measures on every category dataset, supporting the two separately supervised embeddings.

Graph regularization improves single-modality semantic mAP@10 by \GraphLossSemanticMin{} to \GraphLossSemanticMax{} points. Relabeling its edges lowers this measure on every category dataset, showing that the connected objects matter beyond paired supervision. Appendix~\ref{app:ablations} gives all representation, training, and initialization controls.

\subsection{Representation Analysis}
\begin{figure}[!t]
\centering
\includegraphics[width=\linewidth]{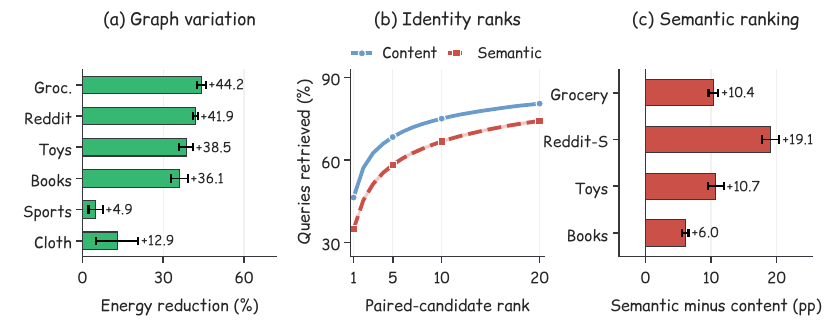}
\caption{GOMA's embeddings at the same weights. Training-edge variation, validation identity ranks, and semantic mAP@10 distinguish their roles. Bars and bands show sample SD.}
\label{fig:signals}
\end{figure}

Fig.~\ref{fig:signals} shows that the semantic embedding reduces training-edge variation and improves semantic mAP@10 by 6.0--19.1 points, while the content embedding ranks more true pairs. The equal-capacity controls support these target-specific roles.

\subsection{Graph Propagation Analysis}
\begin{figure}[!t]
\centering
\includegraphics[width=\linewidth]{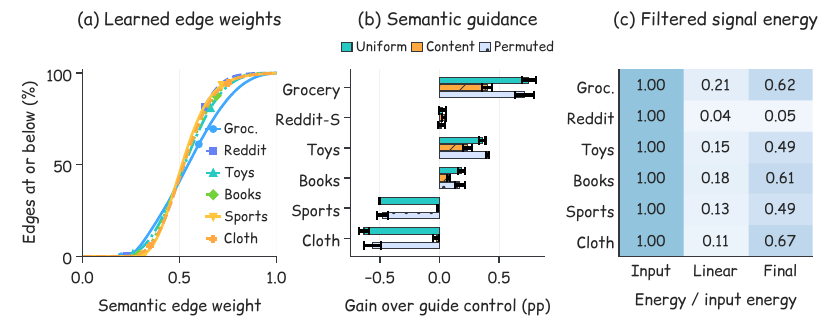}
\caption{Semantic-guided graph propagation. Edge-weight distributions, primary-metric guide comparisons, and relative graph energy cover all six datasets.}
\label{fig:index}
\end{figure}

Fig.~\ref{fig:index}(a) shows heterogeneous edge weights, and panel (b) changes only the guide at the same checkpoint. GOMA improves over uniform, content-based, and permuted-semantic guides on the four category datasets, while Sports and Cloth favor simpler guides. Panel (c) shows that propagation lowers graph energy while retaining alignment with $\mathbf H_0$. Appendix~\ref{app:index_controls} gives complete results.

\section{Conclusion}
Multimodal retrieval over graph-structured collections must preserve object identity while using relationships for cross-object search. We showed that these goals favor different representations and that uniform graph propagation can weaken graph-regularization gains. GOMA therefore learns separate content and semantic embeddings, uses semantics to weight observed edges, and reinjects the initial fused representation during propagation. A task-specific readout selects the representation that matches each relevance definition, and the propagated representations form an offline retrieval index. Across six datasets and four tasks, GOMA leads all 14 primary measures against 14 external methods. Ablations and signal measurements explain the roles of supervision placement, graph regularization, and graph propagation, and show that simpler edge weights can suit relation ranking. Overall, graph information should be organized around the retrieval target rather than forced through one shared output.
 \FloatBarrier
\clearpage
\subsection*{AI use statement}
In this work, we used generative AI tools to aid and polish the writing of this manuscript, including grammar, phrasing, and clarity, as the authors are non-native English speakers. We have reviewed all AI-assisted work, and the authors verified the final text against the reported results and claims. We take responsibility for the final content of this work, including text, claims, or artifacts produced with the aid of generative AI.

\subsection*{Reproducibility statement}
The problem formulation, method, and appendix specify the graph operators, training objective, sampling and model selection, data construction, feature resources, evaluation protocol, and the adaptations applied to external methods. The mathematical analysis states the assumptions behind the reported claims, and the appendix reports the complete retrieval results, ablations, and computational cost. Code will accompany the preprint.

\bibliography{references}
\bibliographystyle{iclr2027_conference}
\clearpage
\appendix
\section*{Appendix Contents}
\noindent
Appendix~\ref{app:related} discusses related work.
Appendix~\ref{app:complete} gives all four retrieval comparisons and extended analyses.
Appendix~\ref{app:theory} derives the graph-signal relations.
Appendix~\ref{app:implementation} specifies the algorithm and optimization.
Appendix~\ref{app:protocol} describes the data and evaluation.
Appendix~\ref{app:baselines} details the external-method adaptations.
Appendix~\ref{app:limitations} summarizes the study's limitations.

\section{Related Work}
\label{app:related}
\subsection{Vision-Language Retrieval}
Paired image-text learning establishes a common space in which corresponding attributes can be compared. VSE++ emphasizes difficult incorrect matches through hard-negative ranking, whereas CLIP uses bidirectional contrastive supervision over large collections of paired examples \citep{faghri2017vse++,clip}. SigLIP uses pairwise sigmoid learning, and SigLIP 2 extends this family with additional training and data-curation procedures \citep{SigLIP,SigLIP22025}. CLIP-Refine studies feature-level modality alignment after pretraining \citep{ClipRefine2025}. These approaches motivate GOMA's use of paired supervision on supplied features. The graph adds relationships among different objects, making the retrieval target a further consideration when deciding how to organize the learned representation.

Universal multimodal retrieval broadens the possible query and document interfaces. UniIR trains instruction-guided retrievers across modality combinations and introduces the M-BEIR benchmark \citep{UniIR2024}. VISTA adds visual information to text embedding models, VLM2Vec trains vision-language models for embedding tasks, and MM-Embed addresses modality bias through modality-aware hard negatives \citep{VISTA2024,VLM2Vec2025,MMEmbed2025}. GME expands multimodal retrieval training through synthesized data, while UniECS studies gated fusion for product search \citep{GME2025,UniECS2025}. Generalized contrastive learning further incorporates richer relevance and ranking objectives \citep{GCLRank2024,GCLUniversal2025}. This line distinguishes retrieval relevance from the recovery of one paired attribute. GOMA additionally considers observed object relationships and explicitly separates attribute-only queries from dual-attribute retrieval over complete records.

Representations may retain information at several semantic granularities. LAPS selects and refines visual patches for patch-word alignment \citep{fu2024linguistic}. AVSE studies asymmetric visual-semantic embeddings, D2S-VSE transfers information between dense and sparse descriptions, and SmartCLIP uses textual masks to select visual concepts \citep{liu2025asymmetric,D2SVSE2025,SmartCLIP2025}. These methods address the unequal content of images and descriptions. GOMA uses object-level feature inputs and distinguishes the desired retrieval outputs through the position of their supervision. It learns image-text correspondence at both embeddings, while graph smoothness acts on the semantic representation used for cross-object relevance.

The distinction between a representation and a projection trained from it is established in contrastive learning. SimCLR shows that the representation before a nonlinear projection can retain information useful for downstream prediction that is less available after projection \citep{SimCLR2020}. GOMA uses this architectural separation for two jointly supervised retrieval outputs. The content embedding is trained for paired identity, the semantic embedding learns graph-regularized cross-modal correspondence, and both remain available for inference. The shared readout comparison in Table~\ref{tab:readout} specifically tests this supervision placement at equal parameter count.

Relations also enter retrieval at different levels. GSMN models structure inside an image-text comparison \citep{gsmn}, and instance-level semantic relationship learning extends supervision across examples \citep{fu2023learning}. NeighborRetr addresses frequently retrieved hubs through neighborhood-aware training \citep{NeighborRetr2025}. IN2R uses intramodal neighbors and graph reasoning to refine noisy cross-modal correspondence \citep{IN2R2026}. These studies involve local phrase structure, learned neighborhoods, or uncertain training pairs. Our setting supplies observed inter-object edges alongside known paired attributes. The graph regularizes semantic outputs during learning and is used for propagation in dual-attribute retrieval.

\subsection{Multimodal Graph Learning}
Multimodal graph benchmarks evaluate topology together with visual and textual object attributes. MM-GRAPH studies visual information in graph learning, MAGB organizes multimodal-attributed graph tasks, and OpenMAG broadens graph-level and modality-level evaluation \citep{mmgraph,MAGB,OpenMAG}. Graph-MLLM studies how multimodal language models can supply encoders, aligners, and predictors for graph tasks \citep{GraphMLLM}. These resources motivate clear definitions of which attributes, relationships, and labels enter a model. Our experiments use common prepared features and retrieval pools so that each adapted method is evaluated on the same supplied information.

Neighborhood aggregation provides a direct way to incorporate a graph. GCN uses normalized aggregation, GAT learns attention over neighbors, and GraphSAGE develops inductive aggregation functions \citep{GCN,GAT,GraphSAGE}. Multimodal recommenders introduce modality-specific graph processing and fusion, as in MMGCN and MGAT, or combine local and global relationships, as in LGMRec \citep{MMGCN,MGAT,LGMRec}. MIG-GT separates modality-specific neighborhoods and includes global interactions \citep{MIG-GT}. BM3 learns through bootstrapped latent views, and MENTOR uses interaction-derived guidance to align modalities \citep{BM32023,MENTOR2025}. These models retain signals useful for preference prediction. GOMA instead trains embeddings whose scores correspond to paired identity or cross-object relevance, using observed object relations for the latter.

Recent methods make modality-topology interactions more selective. DMGC distinguishes homophilous and heterophilous information for clustering, and DGF combines graph filtering with cross-contrastive clustering \citep{DMGC,DGF}. CAMPA separates propagation from aggregation, CoMAG models context-dependent modality-topology alignment, and RoleMAG learns neighbor roles \citep{CAMPA2026,CoMAG2026,RoleMAG2026}. LION introduces Clifford-algebra propagation and adaptive aggregation for multimodal-attributed graphs \citep{LION2026}. These designs change how structural and modality information interact. GOMA uses a normalized graph Dirichlet regularizer on the semantic embedding and keeps identity supervision at the preceding content representation, making retrieval relevance determine which learned output is used.

Multimodal graph pretraining seeks reusable representations across datasets and tasks. UniGraph2 combines modality encoders, alignment, and graph processing across domains \citep{UniGraph2}. Graph4MM incorporates multihop structural information into multimodal learning \citep{graph4mm}. MLaGA and GraphGPT-O connect graphs with multimodal language-model reasoning and generation \citep{MLaGA,GraphGPT-O}. NTSFormer addresses isolated cold-start nodes and missing modalities with self-teaching and attention masks \citep{NTSFormer}. GOMA studies the retrieval interface of a jointly learned model. Independent attribute encoders handle single-modality inputs, and their semantic outputs guide propagation for complete records.

SMGFM is especially close in its distinction between structural consistency and modality-specific detail \citep{SMGFM2026}. It constructs Chebyshev frequency-band tokens and assigns their roles before modality interaction. GOMA organizes the distinction by retrieval relevance and supervision position. Its content and semantic embeddings are jointly learned mappings, and the observed graph penalizes variation in the semantic output. The same semantic signal also guides the weights of observed edges during propagation. This connects retrieval-specific supervision to neighborhood propagation, with the content embedding remaining directly accessible for identity retrieval.

\subsection{Graph Signal Processing and Retrieval Diffusion}
Graph signal processing describes feature variation relative to graph topology through Laplacian energy and spectral filters \citep{Shuman2013}. APPNP combines learned predictions with personalized propagation, and SLOG studies inductive spectral filtering beyond polynomial filters \citep{APPNP2019,SLOG2024}. Reverse message passing and spectral graph pruning modify propagation or topology to mitigate unwanted smoothing \citep{ReverseGNN2024,Pruning2024}. Recent analyses also distinguish the roles of residual connections, normalization, and gradient behavior in deep graph networks \citep{ResidualNorm2025,ResidualMitigate2025,Gradient2024}. GOMA uses these concepts to explain the graph-regularization term on the semantic embedding and the finite graph filter. The learned nonlinear mappings, the linear propagation recurrence, and final row normalization have separate mathematical roles.

Diffusion is an established tool for retrieval. Iscen et al. propagate similarities on a manifold of regional visual descriptors \citep{Iscen2017}. Regularized diffusion studies affinity propagation through higher-order graph relationships \citep{Bai2017}. Matrix Completion on Graphs applies Dirichlet regularization along row and column graphs \citep{Kalofolias2014}. SRFeat combines smoothness and correspondence learning in geometric matching \citep{SRFeat2022}. These works provide precedents for using graph variation to organize matching signals. GOMA connects a graph-regularized semantic embedding to a content embedding for identity retrieval and to propagation over observed object relationships. TWIRLS derives graph aggregation and adaptive edge reweighting from iterative optimization \citep{TWIRLS2021}. GOMA computes continuous edge weights from learned cross-modal semantic agreement, holds them unchanged during propagation, and uses the propagated representations for dual-attribute retrieval. The contribution is this connection between retrieval-specific graph regularization and propagation, alongside independently accessible identity and semantic outputs.

\section{Additional Experimental Results}
\label{app:complete}
\subsection{Complete Retrieval Results}
The following four tables give the complete external comparison for the task definitions in Sec.~\ref{sec:problem}. Every table contains the same 14 external methods and GOMA. GOMA uses its single selected model for all applicable outputs on a dataset, while external methods retain the task-specific selection described in Appendix~\ref{app:baselines}. All repeated-run values are means and sample standard deviations over seeds 43, 44, and 45. The table groups separate vision-language learning, graph encoders, multimodal graph learning, and GOMA.

Table~\ref{tab:complete_identity} reports identity recall in both directions and their bidirectional mean. It extends the primary R@1 comparison to ranks 5 and 10. The strongest bidirectional R@10 on Reddit-S comes from CLIP-Refine, whereas GOMA leads the bidirectional R@1 values used in the main comparison. Directional results expose the image-to-text and text-to-image differences that an average alone can conceal.
\begingroup\completeTableSetup\setlength{\LTcapwidth}{\linewidth}\setlength{\LTleft}{0pt}\setlength{\LTright}{0pt}
\begin{longtable}{@{}p{\linewidth}@{}}
\caption{Complete single-modality identity retrieval. Directional recalls and their bidirectional mean use the same models.}\label{tab:complete_identity}\\
\endfirsthead
\small Table \thetable\ continued \\
\endhead
\multicolumn{1}{r}{\small Continued on the next page} \\
\endfoot
\endlastfoot
\begin{tabular}{@{}p{76pt}*{6}{>{\centering\arraybackslash}p{\dimexpr(\linewidth-76pt-12\tabcolsep)/6\relax}}@{}}
\toprule
\multicolumn{7}{l}{\textbf{Bidirectional mean}} \\
\multirow{2}{*}{Method} & \multicolumn{3}{c}{\textbf{Grocery}} & \multicolumn{3}{c}{\textbf{Reddit-S}} \\
\cmidrule(lr){2-4}\cmidrule(lr){5-7}
 & R@1 & R@5 & R@10 & R@1 & R@5 & R@10 \\
\midrule
\multicolumn{7}{l}{\textbf{Vision-language learning}} \\
VSE++ & \tabstat{49.94}{0.10} & \tabstat{75.53}{0.38} & \tabstat{82.26}{0.26} & \tabstat{16.10}{0.40} & \tabstat{33.98}{0.27} & \tabstat{42.92}{0.20} \\
CLIP-Refine & \tabstat{49.93}{0.07} & \tabstat{76.91}{0.17} & \tabstat{84.02}{0.09} & \tabsecondstat{18.48}{0.11} & \tabsecondstat{37.91}{0.17} & \tabbeststat{47.47}{0.02} \\
SigLIP & \tabsecondstat{52.36}{0.30} & \tabsecondstat{78.05}{0.17} & \tabsecondstat{85.12}{0.05} & \tabstat{17.14}{0.19} & \tabstat{35.89}{0.32} & \tabstat{45.18}{0.30} \\
\midrule
\multicolumn{7}{l}{\textbf{Graph encoders}} \\
GCN & \tabstat{4.60}{0.13} & \tabstat{13.48}{0.68} & \tabstat{19.70}{0.72} & \tabstat{0.97}{0.24} & \tabstat{3.56}{0.20} & \tabstat{5.92}{0.13} \\
GAT & \tabstat{5.97}{1.07} & \tabstat{15.76}{1.53} & \tabstat{22.45}{1.43} & \tabstat{1.20}{0.17} & \tabstat{4.60}{0.24} & \tabstat{7.28}{0.45} \\
GraphSAGE & \tabstat{1.82}{0.38} & \tabstat{6.14}{1.30} & \tabstat{9.73}{2.07} & \tabstat{0.56}{0.16} & \tabstat{1.99}{0.36} & \tabstat{3.24}{0.49} \\
\midrule
\multicolumn{7}{l}{\textbf{Multimodal graph learning}} \\
MMGCN & \tabstat{15.33}{1.13} & \tabstat{37.46}{1.34} & \tabstat{49.32}{1.44} & \tabstat{4.42}{0.24} & \tabstat{13.05}{0.43} & \tabstat{19.88}{0.54} \\
MGAT & \tabstat{20.82}{0.29} & \tabstat{45.53}{0.97} & \tabstat{57.13}{1.07} & \tabstat{5.92}{0.28} & \tabstat{17.16}{0.84} & \tabstat{24.95}{0.71} \\
LGMRec & \tabstat{9.39}{0.60} & \tabstat{24.65}{1.07} & \tabstat{33.80}{1.21} & \tabstat{2.16}{0.41} & \tabstat{7.32}{0.49} & \tabstat{11.38}{0.59} \\
DMGC & \tabstat{24.05}{0.95} & \tabstat{50.78}{1.94} & \tabstat{62.91}{1.43} & \tabstat{6.88}{0.31} & \tabstat{18.91}{0.56} & \tabstat{27.48}{0.21} \\
DGF & \tabstat{20.29}{0.62} & \tabstat{46.31}{0.90} & \tabstat{58.47}{0.77} & \tabstat{5.71}{0.48} & \tabstat{16.82}{0.19} & \tabstat{24.80}{0.12} \\
CAMPA & \tabstat{3.25}{1.15} & \tabstat{9.57}{2.98} & \tabstat{14.25}{3.79} & \tabstat{0.73}{0.14} & \tabstat{2.69}{0.46} & \tabstat{4.34}{0.81} \\
CoMAG & \tabstat{12.15}{0.87} & \tabstat{29.08}{1.36} & \tabstat{38.14}{1.72} & \tabstat{2.34}{0.34} & \tabstat{7.26}{0.32} & \tabstat{10.72}{0.81} \\
RoleMAG & \tabstat{5.57}{0.25} & \tabstat{15.54}{0.59} & \tabstat{22.14}{0.75} & \tabstat{0.98}{0.11} & \tabstat{3.60}{0.17} & \tabstat{5.71}{0.27} \\
\midrule
\multicolumn{7}{l}{\textbf{Proposed method}} \\
\textbf{GOMA} & \tabbeststat{52.96}{0.06} & \tabbeststat{78.57}{0.06} & \tabbeststat{85.14}{0.17} & \tabbeststat{18.66}{0.19} & \tabbeststat{38.13}{0.26} & \tabsecondstat{47.37}{0.36} \\
\bottomrule
\end{tabular}\\
\begin{tabular}{@{}p{76pt}*{6}{>{\centering\arraybackslash}p{\dimexpr(\linewidth-76pt-12\tabcolsep)/6\relax}}@{}}
\toprule
\multicolumn{7}{l}{\textbf{Bidirectional mean}} \\
\multirow{2}{*}{Method} & \multicolumn{3}{c}{\textbf{Toys}} & \multicolumn{3}{c}{\textbf{Books}} \\
\cmidrule(lr){2-4}\cmidrule(lr){5-7}
 & R@1 & R@5 & R@10 & R@1 & R@5 & R@10 \\
\midrule
\multicolumn{7}{l}{\textbf{Vision-language learning}} \\
VSE++ & \tabstat{59.19}{0.12} & \tabstat{82.20}{0.10} & \tabstat{87.49}{0.08} & \tabsecondstat{39.23}{0.09} & \tabsecondstat{61.23}{0.09} & \tabsecondstat{68.98}{0.03} \\
CLIP-Refine & \tabsecondstat{62.57}{0.10} & \tabsecondstat{85.34}{0.05} & \tabsecondstat{90.22}{0.04} & \tabstat{35.44}{0.01} & \tabstat{56.33}{0.02} & \tabstat{64.15}{0.04} \\
SigLIP & \tabstat{61.29}{0.49} & \tabstat{84.78}{0.21} & \tabstat{90.12}{0.18} & \tabstat{37.84}{0.07} & \tabstat{59.87}{0.14} & \tabstat{67.76}{0.08} \\
\midrule
\multicolumn{7}{l}{\textbf{Graph encoders}} \\
GCN & \tabstat{2.87}{0.31} & \tabstat{8.48}{0.39} & \tabstat{12.54}{0.41} & \tabstat{0.32}{0.04} & \tabstat{1.15}{0.02} & \tabstat{1.95}{0.05} \\
GAT & \tabstat{4.64}{0.78} & \tabstat{12.76}{0.95} & \tabstat{18.18}{0.93} & \tabstat{0.91}{0.19} & \tabstat{3.26}{0.42} & \tabstat{5.30}{0.66} \\
GraphSAGE & \tabstat{1.17}{0.21} & \tabstat{4.00}{0.67} & \tabstat{6.57}{1.10} & \tabstat{0.38}{0.07} & \tabstat{1.52}{0.15} & \tabstat{2.64}{0.23} \\
\midrule
\multicolumn{7}{l}{\textbf{Multimodal graph learning}} \\
MMGCN & \tabstat{14.26}{0.97} & \tabstat{33.03}{1.95} & \tabstat{42.32}{2.22} & \tabstat{12.44}{0.32} & \tabstat{28.36}{0.46} & \tabstat{37.18}{0.54} \\
MGAT & \tabstat{22.81}{0.73} & \tabstat{45.53}{0.64} & \tabstat{55.90}{0.56} & \tabstat{17.23}{0.54} & \tabstat{35.48}{0.90} & \tabstat{44.48}{1.14} \\
LGMRec & \tabstat{8.61}{0.11} & \tabstat{21.17}{0.15} & \tabstat{28.33}{0.10} & \tabstat{1.54}{0.02} & \tabstat{5.12}{0.02} & \tabstat{8.05}{0.05} \\
DMGC & \tabstat{23.74}{0.68} & \tabstat{46.28}{0.35} & \tabstat{55.41}{0.41} & \tabstat{10.36}{0.47} & \tabstat{24.02}{0.58} & \tabstat{31.87}{0.75} \\
DGF & \tabstat{20.81}{0.41} & \tabstat{44.26}{0.79} & \tabstat{55.18}{0.78} & \tabstat{13.11}{0.52} & \tabstat{29.35}{0.73} & \tabstat{38.31}{0.71} \\
CAMPA & \tabstat{3.69}{0.75} & \tabstat{10.42}{1.72} & \tabstat{14.78}{2.33} & \tabstat{0.24}{0.06} & \tabstat{0.87}{0.23} & \tabstat{1.47}{0.34} \\
CoMAG & \tabstat{11.20}{0.91} & \tabstat{25.83}{1.04} & \tabstat{33.85}{0.85} & \tabstat{1.46}{0.16} & \tabstat{4.74}{0.40} & \tabstat{7.35}{0.50} \\
RoleMAG & \tabstat{4.07}{0.89} & \tabstat{10.67}{1.22} & \tabstat{15.65}{1.36} & \tabstat{0.61}{0.06} & \tabstat{2.36}{0.25} & \tabstat{3.84}{0.30} \\
\midrule
\multicolumn{7}{l}{\textbf{Proposed method}} \\
\textbf{GOMA} & \tabbeststat{62.94}{0.40} & \tabbeststat{86.09}{0.32} & \tabbeststat{90.96}{0.20} & \tabbeststat{40.14}{0.31} & \tabbeststat{62.22}{0.42} & \tabbeststat{69.84}{0.39} \\
\bottomrule
\end{tabular}\\
\begin{tabular}{@{}p{76pt}*{6}{>{\centering\arraybackslash}p{\dimexpr(\linewidth-76pt-12\tabcolsep)/6\relax}}@{}}
\toprule
\multicolumn{7}{l}{\textbf{Image to text}} \\
\multirow{2}{*}{Method} & \multicolumn{3}{c}{\textbf{Grocery}} & \multicolumn{3}{c}{\textbf{Reddit-S}} \\
\cmidrule(lr){2-4}\cmidrule(lr){5-7}
 & R@1 & R@5 & R@10 & R@1 & R@5 & R@10 \\
\midrule
\multicolumn{7}{l}{\textbf{Vision-language learning}} \\
VSE++ & \tabstat{49.80}{0.28} & \tabstat{75.67}{0.30} & \tabstat{82.12}{0.17} & \tabstat{16.21}{0.46} & \tabstat{34.30}{0.05} & \tabstat{43.28}{0.14} \\
CLIP-Refine & \tabstat{50.05}{0.08} & \tabstat{76.86}{0.20} & \tabsecondstat{84.22}{0.27} & \tabbeststat{18.63}{0.08} & \tabsecondstat{37.86}{0.26} & \tabsecondstat{47.15}{0.10} \\
SigLIP & \tabsecondstat{52.53}{0.49} & \tabsecondstat{77.98}{0.09} & \tabbeststat{85.14}{0.08} & \tabstat{17.27}{0.26} & \tabstat{35.86}{0.71} & \tabstat{45.27}{0.49} \\
\midrule
\multicolumn{7}{l}{\textbf{Graph encoders}} \\
GCN & \tabstat{4.58}{0.34} & \tabstat{13.64}{0.78} & \tabstat{19.89}{0.79} & \tabstat{0.99}{0.38} & \tabstat{3.59}{0.16} & \tabstat{6.13}{0.13} \\
GAT & \tabstat{6.05}{1.14} & \tabstat{15.65}{1.63} & \tabstat{22.50}{1.33} & \tabstat{1.15}{0.08} & \tabstat{4.41}{0.28} & \tabstat{7.08}{0.56} \\
GraphSAGE & \tabstat{1.89}{0.54} & \tabstat{6.13}{1.41} & \tabstat{9.73}{2.23} & \tabstat{0.48}{0.05} & \tabstat{1.82}{0.16} & \tabstat{3.02}{0.45} \\
\midrule
\multicolumn{7}{l}{\textbf{Multimodal graph learning}} \\
MMGCN & \tabstat{15.24}{1.00} & \tabstat{36.94}{1.00} & \tabstat{48.96}{1.70} & \tabstat{4.52}{0.26} & \tabstat{13.40}{0.32} & \tabstat{20.23}{0.17} \\
MGAT & \tabstat{20.13}{0.31} & \tabstat{45.39}{0.84} & \tabstat{56.90}{1.15} & \tabstat{6.14}{0.21} & \tabstat{17.42}{0.91} & \tabstat{25.23}{1.03} \\
LGMRec & \tabstat{9.21}{0.76} & \tabstat{24.35}{1.32} & \tabstat{33.49}{1.69} & \tabstat{2.06}{0.28} & \tabstat{6.91}{0.41} & \tabstat{11.03}{0.54} \\
DMGC & \tabstat{23.57}{1.01} & \tabstat{50.20}{1.83} & \tabstat{62.75}{1.12} & \tabstat{6.97}{0.36} & \tabstat{19.29}{0.25} & \tabstat{27.79}{0.19} \\
DGF & \tabstat{19.94}{0.68} & \tabstat{45.82}{1.26} & \tabstat{58.24}{1.06} & \tabstat{5.91}{0.62} & \tabstat{17.15}{0.23} & \tabstat{25.18}{0.11} \\
CAMPA & \tabstat{3.45}{2.93} & \tabstat{10.55}{7.78} & \tabstat{15.44}{10.28} & \tabstat{0.43}{0.18} & \tabstat{1.97}{0.54} & \tabstat{3.30}{0.94} \\
CoMAG & \tabstat{12.69}{0.37} & \tabstat{30.21}{1.28} & \tabstat{39.33}{1.18} & \tabstat{2.47}{0.41} & \tabstat{6.99}{0.35} & \tabstat{10.32}{0.79} \\
RoleMAG & \tabstat{5.35}{0.13} & \tabstat{15.51}{0.71} & \tabstat{21.83}{0.74} & \tabstat{0.89}{0.12} & \tabstat{3.59}{0.07} & \tabstat{5.80}{0.05} \\
\midrule
\multicolumn{7}{l}{\textbf{Proposed method}} \\
\textbf{GOMA} & \tabbeststat{52.85}{0.06} & \tabbeststat{78.65}{0.22} & \tabbeststat{85.14}{0.15} & \tabsecondstat{18.51}{0.18} & \tabbeststat{38.36}{0.30} & \tabbeststat{47.79}{0.30} \\
\bottomrule
\end{tabular}\\
\begin{tabular}{@{}p{76pt}*{6}{>{\centering\arraybackslash}p{\dimexpr(\linewidth-76pt-12\tabcolsep)/6\relax}}@{}}
\toprule
\multicolumn{7}{l}{\textbf{Image to text}} \\
\multirow{2}{*}{Method} & \multicolumn{3}{c}{\textbf{Toys}} & \multicolumn{3}{c}{\textbf{Books}} \\
\cmidrule(lr){2-4}\cmidrule(lr){5-7}
 & R@1 & R@5 & R@10 & R@1 & R@5 & R@10 \\
\midrule
\multicolumn{7}{l}{\textbf{Vision-language learning}} \\
VSE++ & \tabstat{59.21}{0.17} & \tabstat{81.84}{0.12} & \tabstat{87.20}{0.03} & \tabsecondstat{40.02}{0.13} & \tabsecondstat{61.79}{0.04} & \tabsecondstat{69.38}{0.12} \\
CLIP-Refine & \tabsecondstat{62.70}{0.18} & \tabsecondstat{85.05}{0.07} & \tabsecondstat{90.08}{0.04} & \tabstat{36.75}{0.05} & \tabstat{57.45}{0.03} & \tabstat{65.17}{0.03} \\
SigLIP & \tabstat{61.52}{0.29} & \tabstat{84.34}{0.23} & \tabstat{89.59}{0.23} & \tabstat{38.82}{0.11} & \tabstat{60.62}{0.18} & \tabstat{68.17}{0.10} \\
\midrule
\multicolumn{7}{l}{\textbf{Graph encoders}} \\
GCN & \tabstat{2.77}{0.39} & \tabstat{8.25}{0.40} & \tabstat{12.35}{0.35} & \tabstat{0.37}{0.04} & \tabstat{1.24}{0.04} & \tabstat{2.08}{0.07} \\
GAT & \tabstat{4.63}{0.84} & \tabstat{12.59}{1.13} & \tabstat{18.02}{0.83} & \tabstat{0.98}{0.16} & \tabstat{3.48}{0.36} & \tabstat{5.60}{0.64} \\
GraphSAGE & \tabstat{1.23}{0.16} & \tabstat{3.97}{0.80} & \tabstat{6.48}{1.21} & \tabstat{0.42}{0.06} & \tabstat{1.64}{0.18} & \tabstat{2.79}{0.21} \\
\midrule
\multicolumn{7}{l}{\textbf{Multimodal graph learning}} \\
MMGCN & \tabstat{14.19}{1.08} & \tabstat{33.00}{1.70} & \tabstat{42.14}{2.10} & \tabstat{12.60}{0.23} & \tabstat{28.45}{0.39} & \tabstat{37.38}{0.65} \\
MGAT & \tabstat{23.16}{0.57} & \tabstat{45.66}{0.56} & \tabstat{56.13}{0.44} & \tabstat{17.43}{0.59} & \tabstat{35.67}{0.85} & \tabstat{44.65}{0.97} \\
LGMRec & \tabstat{8.51}{0.17} & \tabstat{21.11}{0.07} & \tabstat{28.19}{0.19} & \tabstat{1.67}{0.09} & \tabstat{5.49}{0.03} & \tabstat{8.59}{0.14} \\
DMGC & \tabstat{23.72}{0.84} & \tabstat{46.04}{0.13} & \tabstat{55.16}{0.19} & \tabstat{10.38}{0.56} & \tabstat{24.16}{0.71} & \tabstat{32.03}{0.81} \\
DGF & \tabstat{21.04}{0.45} & \tabstat{44.16}{0.99} & \tabstat{55.05}{0.91} & \tabstat{13.16}{0.51} & \tabstat{29.50}{0.81} & \tabstat{38.57}{0.78} \\
CAMPA & \tabstat{3.99}{1.64} & \tabstat{11.55}{4.24} & \tabstat{16.22}{5.57} & \tabstat{0.28}{0.03} & \tabstat{1.01}{0.14} & \tabstat{1.71}{0.23} \\
CoMAG & \tabstat{11.77}{1.01} & \tabstat{26.82}{0.95} & \tabstat{34.92}{0.86} & \tabstat{1.46}{0.23} & \tabstat{4.75}{0.55} & \tabstat{7.31}{0.72} \\
RoleMAG & \tabstat{3.93}{0.95} & \tabstat{10.51}{1.15} & \tabstat{15.41}{1.30} & \tabstat{0.65}{0.04} & \tabstat{2.43}{0.21} & \tabstat{3.95}{0.24} \\
\midrule
\multicolumn{7}{l}{\textbf{Proposed method}} \\
\textbf{GOMA} & \tabbeststat{63.11}{0.45} & \tabbeststat{85.60}{0.42} & \tabbeststat{90.57}{0.31} & \tabbeststat{41.13}{0.35} & \tabbeststat{62.93}{0.38} & \tabbeststat{70.40}{0.42} \\
\bottomrule
\end{tabular}\\
\begin{tabular}{@{}p{76pt}*{6}{>{\centering\arraybackslash}p{\dimexpr(\linewidth-76pt-12\tabcolsep)/6\relax}}@{}}
\toprule
\multicolumn{7}{l}{\textbf{Text to image}} \\
\multirow{2}{*}{Method} & \multicolumn{3}{c}{\textbf{Grocery}} & \multicolumn{3}{c}{\textbf{Reddit-S}} \\
\cmidrule(lr){2-4}\cmidrule(lr){5-7}
 & R@1 & R@5 & R@10 & R@1 & R@5 & R@10 \\
\midrule
\multicolumn{7}{l}{\textbf{Vision-language learning}} \\
VSE++ & \tabstat{50.08}{0.25} & \tabstat{75.39}{0.50} & \tabstat{82.40}{0.39} & \tabstat{15.98}{0.37} & \tabstat{33.66}{0.50} & \tabstat{42.56}{0.38} \\
CLIP-Refine & \tabstat{49.81}{0.10} & \tabstat{76.97}{0.21} & \tabstat{83.81}{0.11} & \tabsecondstat{18.33}{0.23} & \tabbeststat{37.95}{0.13} & \tabbeststat{47.79}{0.08} \\
SigLIP & \tabsecondstat{52.20}{0.16} & \tabsecondstat{78.11}{0.31} & \tabsecondstat{85.11}{0.06} & \tabstat{17.01}{0.25} & \tabstat{35.92}{0.23} & \tabstat{45.09}{0.12} \\
\midrule
\multicolumn{7}{l}{\textbf{Graph encoders}} \\
GCN & \tabstat{4.61}{0.08} & \tabstat{13.32}{0.62} & \tabstat{19.52}{0.79} & \tabstat{0.96}{0.17} & \tabstat{3.52}{0.26} & \tabstat{5.71}{0.18} \\
GAT & \tabstat{5.89}{1.08} & \tabstat{15.86}{1.42} & \tabstat{22.40}{1.55} & \tabstat{1.25}{0.28} & \tabstat{4.79}{0.21} & \tabstat{7.48}{0.33} \\
GraphSAGE & \tabstat{1.75}{0.23} & \tabstat{6.16}{1.20} & \tabstat{9.72}{1.93} & \tabstat{0.64}{0.27} & \tabstat{2.16}{0.55} & \tabstat{3.45}{0.53} \\
\midrule
\multicolumn{7}{l}{\textbf{Multimodal graph learning}} \\
MMGCN & \tabstat{15.42}{1.36} & \tabstat{37.98}{1.69} & \tabstat{49.68}{1.27} & \tabstat{4.31}{0.22} & \tabstat{12.70}{0.60} & \tabstat{19.52}{0.92} \\
MGAT & \tabstat{21.50}{0.50} & \tabstat{45.67}{1.10} & \tabstat{57.35}{1.06} & \tabstat{5.69}{0.46} & \tabstat{16.89}{0.77} & \tabstat{24.67}{0.54} \\
LGMRec & \tabstat{9.56}{0.62} & \tabstat{24.95}{0.84} & \tabstat{34.12}{0.75} & \tabstat{2.26}{0.57} & \tabstat{7.73}{0.58} & \tabstat{11.74}{0.71} \\
DMGC & \tabstat{24.53}{0.88} & \tabstat{51.37}{2.11} & \tabstat{63.06}{1.75} & \tabstat{6.79}{0.56} & \tabstat{18.54}{0.93} & \tabstat{27.17}{0.54} \\
DGF & \tabstat{20.65}{0.59} & \tabstat{46.80}{0.60} & \tabstat{58.71}{0.54} & \tabstat{5.50}{0.37} & \tabstat{16.49}{0.15} & \tabstat{24.42}{0.28} \\
CAMPA & \tabstat{3.06}{0.79} & \tabstat{8.60}{2.16} & \tabstat{13.06}{2.96} & \tabstat{1.02}{0.37} & \tabstat{3.42}{1.22} & \tabstat{5.38}{1.93} \\
CoMAG & \tabstat{11.61}{1.38} & \tabstat{27.95}{1.44} & \tabstat{36.95}{2.33} & \tabstat{2.20}{0.28} & \tabstat{7.53}{0.70} & \tabstat{11.12}{1.03} \\
RoleMAG & \tabstat{5.79}{0.59} & \tabstat{15.57}{0.55} & \tabstat{22.44}{0.81} & \tabstat{1.07}{0.16} & \tabstat{3.61}{0.29} & \tabstat{5.63}{0.52} \\
\midrule
\multicolumn{7}{l}{\textbf{Proposed method}} \\
\textbf{GOMA} & \tabbeststat{53.07}{0.15} & \tabbeststat{78.50}{0.16} & \tabbeststat{85.14}{0.21} & \tabbeststat{18.80}{0.25} & \tabsecondstat{37.90}{0.24} & \tabsecondstat{46.96}{0.42} \\
\bottomrule
\end{tabular}\\
\begin{tabular}{@{}p{76pt}*{6}{>{\centering\arraybackslash}p{\dimexpr(\linewidth-76pt-12\tabcolsep)/6\relax}}@{}}
\toprule
\multicolumn{7}{l}{\textbf{Text to image}} \\
\multirow{2}{*}{Method} & \multicolumn{3}{c}{\textbf{Toys}} & \multicolumn{3}{c}{\textbf{Books}} \\
\cmidrule(lr){2-4}\cmidrule(lr){5-7}
 & R@1 & R@5 & R@10 & R@1 & R@5 & R@10 \\
\midrule
\multicolumn{7}{l}{\textbf{Vision-language learning}} \\
VSE++ & \tabstat{59.17}{0.06} & \tabstat{82.57}{0.09} & \tabstat{87.78}{0.14} & \tabsecondstat{38.44}{0.14} & \tabsecondstat{60.67}{0.17} & \tabsecondstat{68.58}{0.08} \\
CLIP-Refine & \tabsecondstat{62.44}{0.11} & \tabsecondstat{85.63}{0.16} & \tabstat{90.36}{0.04} & \tabstat{34.14}{0.03} & \tabstat{55.20}{0.04} & \tabstat{63.14}{0.06} \\
SigLIP & \tabstat{61.05}{0.69} & \tabstat{85.21}{0.19} & \tabsecondstat{90.65}{0.13} & \tabstat{36.88}{0.23} & \tabstat{59.13}{0.11} & \tabstat{67.36}{0.12} \\
\midrule
\multicolumn{7}{l}{\textbf{Graph encoders}} \\
GCN & \tabstat{2.97}{0.25} & \tabstat{8.70}{0.39} & \tabstat{12.74}{0.51} & \tabstat{0.26}{0.03} & \tabstat{1.05}{0.01} & \tabstat{1.83}{0.05} \\
GAT & \tabstat{4.64}{0.72} & \tabstat{12.93}{0.79} & \tabstat{18.34}{1.05} & \tabstat{0.85}{0.22} & \tabstat{3.03}{0.51} & \tabstat{5.00}{0.71} \\
GraphSAGE & \tabstat{1.12}{0.27} & \tabstat{4.02}{0.54} & \tabstat{6.66}{1.02} & \tabstat{0.35}{0.09} & \tabstat{1.40}{0.14} & \tabstat{2.50}{0.24} \\
\midrule
\multicolumn{7}{l}{\textbf{Multimodal graph learning}} \\
MMGCN & \tabstat{14.33}{0.87} & \tabstat{33.06}{2.21} & \tabstat{42.50}{2.35} & \tabstat{12.26}{0.43} & \tabstat{28.27}{0.55} & \tabstat{36.98}{0.52} \\
MGAT & \tabstat{22.46}{0.88} & \tabstat{45.41}{0.81} & \tabstat{55.68}{0.69} & \tabstat{17.03}{0.48} & \tabstat{35.29}{0.97} & \tabstat{44.30}{1.31} \\
LGMRec & \tabstat{8.70}{0.37} & \tabstat{21.24}{0.35} & \tabstat{28.47}{0.01} & \tabstat{1.42}{0.09} & \tabstat{4.75}{0.05} & \tabstat{7.50}{0.04} \\
DMGC & \tabstat{23.76}{0.57} & \tabstat{46.51}{0.74} & \tabstat{55.67}{0.65} & \tabstat{10.33}{0.38} & \tabstat{23.87}{0.44} & \tabstat{31.72}{0.69} \\
DGF & \tabstat{20.58}{0.39} & \tabstat{44.36}{0.59} & \tabstat{55.31}{0.68} & \tabstat{13.07}{0.54} & \tabstat{29.19}{0.67} & \tabstat{38.05}{0.65} \\
CAMPA & \tabstat{3.40}{1.92} & \tabstat{9.30}{4.02} & \tabstat{13.33}{5.68} & \tabstat{0.19}{0.09} & \tabstat{0.73}{0.32} & \tabstat{1.23}{0.51} \\
CoMAG & \tabstat{10.62}{0.91} & \tabstat{24.84}{1.13} & \tabstat{32.78}{0.86} & \tabstat{1.46}{0.20} & \tabstat{4.74}{0.43} & \tabstat{7.40}{0.66} \\
RoleMAG & \tabstat{4.22}{0.84} & \tabstat{10.83}{1.30} & \tabstat{15.89}{1.41} & \tabstat{0.57}{0.09} & \tabstat{2.30}{0.31} & \tabstat{3.73}{0.35} \\
\midrule
\multicolumn{7}{l}{\textbf{Proposed method}} \\
\textbf{GOMA} & \tabbeststat{62.77}{0.36} & \tabbeststat{86.58}{0.27} & \tabbeststat{91.35}{0.09} & \tabbeststat{39.15}{0.27} & \tabbeststat{61.51}{0.46} & \tabbeststat{69.29}{0.36} \\
\bottomrule
\end{tabular}\\
\end{longtable}
\endgroup
 
Table~\ref{tab:complete_single} extends single-modality semantic retrieval to precision and semantic recall. GOMA has the highest mean for every displayed metric on all four datasets. The query and gallery remain disjoint, and each attribute is encoded independently, so the results concern cross-object semantic ranking from a single modality.
\begingroup\completeTableSetup\setlength{\LTcapwidth}{\linewidth}\setlength{\LTleft}{0pt}\setlength{\LTright}{0pt}
\begin{longtable}{@{}p{\linewidth}@{}}
\caption{Complete single-modality cross-object semantic retrieval. R@1 uses semantic relevance.}\label{tab:complete_single}\\
\endfirsthead
\small Table \thetable\ continued \\
\endhead
\multicolumn{1}{r}{\small Continued on the next page} \\
\endfoot
\endlastfoot
\begin{tabular}{@{}p{76pt}*{6}{>{\centering\arraybackslash}p{\dimexpr(\linewidth-76pt-12\tabcolsep)/6\relax}}@{}}
\toprule
\multirow{2}{*}{Method} & \multicolumn{3}{c}{\textbf{Grocery}} & \multicolumn{3}{c}{\textbf{Reddit-S}} \\
\cmidrule(lr){2-4}\cmidrule(lr){5-7}
 & mAP@10 & P@10 & Sem. R@1 & mAP@10 & P@10 & Sem. R@1 \\
\midrule
\multicolumn{7}{l}{\textbf{Vision-language learning}} \\
VSE++ & \tabstat{30.53}{0.14} & \tabstat{39.26}{0.10} & \tabstat{55.54}{0.32} & \tabstat{37.34}{0.26} & \tabstat{46.21}{0.28} & \tabstat{59.52}{0.11} \\
CLIP-Refine & \tabstat{33.72}{0.06} & \tabstat{42.18}{0.06} & \tabstat{58.66}{0.10} & \tabstat{46.19}{0.10} & \tabstat{54.02}{0.09} & \tabstat{64.65}{0.14} \\
SigLIP & \tabsecondstat{38.27}{0.12} & \tabsecondstat{46.31}{0.10} & \tabsecondstat{59.52}{0.13} & \tabsecondstat{50.89}{0.04} & \tabsecondstat{58.29}{0.08} & \tabsecondstat{66.59}{0.42} \\
\midrule
\multicolumn{7}{l}{\textbf{Graph encoders}} \\
GCN & \tabstat{8.86}{0.53} & \tabstat{15.73}{0.62} & \tabstat{19.54}{1.22} & \tabstat{11.16}{1.20} & \tabstat{18.76}{1.32} & \tabstat{22.05}{1.55} \\
GAT & \tabstat{11.96}{0.42} & \tabstat{19.16}{0.46} & \tabstat{23.43}{0.75} & \tabstat{16.41}{1.46} & \tabstat{24.65}{1.78} & \tabstat{28.87}{2.21} \\
GraphSAGE & \tabstat{6.78}{0.45} & \tabstat{12.81}{0.64} & \tabstat{14.60}{1.34} & \tabstat{10.17}{1.26} & \tabstat{16.96}{1.52} & \tabstat{18.33}{1.45} \\
\midrule
\multicolumn{7}{l}{\textbf{Multimodal graph learning}} \\
MMGCN & \tabstat{28.15}{0.53} & \tabstat{36.57}{0.61} & \tabstat{45.53}{1.42} & \tabstat{42.22}{0.36} & \tabstat{50.13}{0.37} & \tabstat{55.40}{0.30} \\
MGAT & \tabstat{30.54}{1.00} & \tabstat{38.66}{0.98} & \tabstat{49.62}{1.23} & \tabstat{43.60}{0.98} & \tabstat{51.39}{0.72} & \tabstat{58.49}{0.95} \\
LGMRec & \tabstat{12.88}{0.23} & \tabstat{20.83}{0.36} & \tabstat{26.18}{0.58} & \tabstat{18.79}{0.39} & \tabstat{27.15}{0.30} & \tabstat{31.45}{1.02} \\
DMGC & \tabstat{34.70}{0.23} & \tabstat{42.67}{0.27} & \tabstat{53.45}{0.39} & \tabstat{50.60}{0.27} & \tabstat{57.26}{0.23} & \tabstat{61.56}{0.47} \\
DGF & \tabstat{32.45}{0.96} & \tabstat{40.58}{0.86} & \tabstat{50.41}{0.57} & \tabstat{44.62}{0.97} & \tabstat{52.38}{0.95} & \tabstat{58.20}{0.54} \\
CAMPA & \tabstat{8.76}{0.78} & \tabstat{14.93}{1.19} & \tabstat{18.23}{1.65} & \tabstat{10.52}{1.77} & \tabstat{15.78}{1.96} & \tabstat{17.75}{1.56} \\
CoMAG & \tabstat{13.84}{0.75} & \tabstat{21.69}{0.70} & \tabstat{29.38}{1.72} & \tabstat{17.05}{1.51} & \tabstat{25.12}{1.76} & \tabstat{30.38}{1.73} \\
RoleMAG & \tabstat{9.32}{0.31} & \tabstat{16.51}{0.44} & \tabstat{20.87}{0.80} & \tabstat{13.06}{0.65} & \tabstat{20.72}{0.21} & \tabstat{24.95}{0.74} \\
\midrule
\multicolumn{7}{l}{\textbf{Proposed method}} \\
\textbf{GOMA} & \tabbeststat{42.82}{0.22} & \tabbeststat{50.23}{0.17} & \tabbeststat{60.16}{0.39} & \tabbeststat{59.74}{0.26} & \tabbeststat{65.39}{0.31} & \tabbeststat{70.13}{0.08} \\
\bottomrule
\end{tabular}\\
\begin{tabular}{@{}p{76pt}*{6}{>{\centering\arraybackslash}p{\dimexpr(\linewidth-76pt-12\tabcolsep)/6\relax}}@{}}
\toprule
\multirow{2}{*}{Method} & \multicolumn{3}{c}{\textbf{Toys}} & \multicolumn{3}{c}{\textbf{Books}} \\
\cmidrule(lr){2-4}\cmidrule(lr){5-7}
 & mAP@10 & P@10 & Sem. R@1 & mAP@10 & P@10 & Sem. R@1 \\
\midrule
\multicolumn{7}{l}{\textbf{Vision-language learning}} \\
VSE++ & \tabstat{23.09}{0.09} & \tabstat{32.35}{0.09} & \tabstat{50.78}{0.46} & \tabstat{39.99}{0.29} & \tabstat{50.51}{0.23} & \tabstat{55.66}{0.28} \\
CLIP-Refine & \tabstat{27.57}{0.05} & \tabstat{37.01}{0.08} & \tabstat{54.50}{0.04} & \tabsecondstat{42.82}{0.02} & \tabsecondstat{52.98}{0.02} & \tabsecondstat{57.32}{0.06} \\
SigLIP & \tabsecondstat{32.75}{0.25} & \tabsecondstat{42.04}{0.33} & \tabsecondstat{57.46}{0.02} & \tabstat{42.42}{0.07} & \tabstat{52.57}{0.01} & \tabstat{57.10}{0.10} \\
\midrule
\multicolumn{7}{l}{\textbf{Graph encoders}} \\
GCN & \tabstat{4.55}{0.06} & \tabstat{10.26}{0.07} & \tabstat{11.43}{0.60} & \tabstat{16.65}{0.99} & \tabstat{26.78}{1.08} & \tabstat{27.87}{1.25} \\
GAT & \tabstat{5.08}{0.49} & \tabstat{11.08}{0.78} & \tabstat{12.71}{1.20} & \tabstat{23.08}{0.67} & \tabstat{33.81}{0.93} & \tabstat{35.45}{0.84} \\
GraphSAGE & \tabstat{3.46}{0.17} & \tabstat{8.55}{0.38} & \tabstat{8.85}{0.60} & \tabstat{17.29}{1.20} & \tabstat{27.67}{1.39} & \tabstat{29.05}{1.31} \\
\midrule
\multicolumn{7}{l}{\textbf{Multimodal graph learning}} \\
MMGCN & \tabstat{16.94}{0.76} & \tabstat{25.50}{0.73} & \tabstat{35.20}{1.25} & \tabstat{40.82}{0.73} & \tabstat{50.91}{0.66} & \tabstat{53.91}{0.55} \\
MGAT & \tabstat{20.51}{0.31} & \tabstat{29.17}{0.22} & \tabstat{41.13}{0.15} & \tabstat{41.71}{0.12} & \tabstat{51.61}{0.09} & \tabstat{55.16}{0.17} \\
LGMRec & \tabstat{6.50}{0.43} & \tabstat{13.26}{0.57} & \tabstat{16.64}{0.40} & \tabstat{25.63}{0.98} & \tabstat{36.54}{0.98} & \tabstat{38.27}{1.10} \\
DMGC & \tabstat{17.48}{0.35} & \tabstat{26.10}{0.43} & \tabstat{36.02}{0.58} & \tabstat{42.14}{0.37} & \tabstat{52.03}{0.30} & \tabstat{54.66}{0.48} \\
DGF & \tabstat{23.01}{0.64} & \tabstat{31.86}{0.37} & \tabstat{42.65}{0.48} & \tabstat{40.89}{0.27} & \tabstat{50.99}{0.24} & \tabstat{54.31}{0.16} \\
CAMPA & \tabstat{5.37}{0.52} & \tabstat{11.03}{0.79} & \tabstat{13.27}{1.10} & \tabstat{16.07}{1.49} & \tabstat{24.84}{2.53} & \tabstat{26.32}{2.75} \\
CoMAG & \tabstat{7.18}{0.22} & \tabstat{14.20}{0.25} & \tabstat{19.68}{0.35} & \tabstat{26.37}{0.61} & \tabstat{37.11}{0.53} & \tabstat{38.94}{0.31} \\
RoleMAG & \tabstat{4.93}{0.33} & \tabstat{10.77}{0.65} & \tabstat{12.93}{0.84} & \tabstat{20.76}{0.85} & \tabstat{31.21}{0.78} & \tabstat{32.67}{0.89} \\
\midrule
\multicolumn{7}{l}{\textbf{Proposed method}} \\
\textbf{GOMA} & \tabbeststat{41.52}{0.44} & \tabbeststat{50.22}{0.48} & \tabbeststat{61.36}{0.30} & \tabbeststat{46.16}{0.03} & \tabbeststat{55.79}{0.01} & \tabbeststat{59.79}{0.07} \\
\bottomrule
\end{tabular}\\
\end{longtable}
\endgroup
 
Table~\ref{tab:complete_index} reports the same semantic metrics when complete object records and the permitted observed graph are available. The relevance definition is unchanged from Table~\ref{tab:complete_single}, allowing the two interfaces to be compared without confusing a semantic match with paired identity. GOMA leads all displayed means in this dual-attribute comparison.
\begingroup\completeTableSetup\setlength{\LTcapwidth}{\linewidth}\setlength{\LTleft}{0pt}\setlength{\LTright}{0pt}
\begin{longtable}{@{}p{\linewidth}@{}}
\caption{Complete dual-attribute cross-object semantic retrieval. R@1 uses semantic relevance.}\label{tab:complete_index}\\
\endfirsthead
\small Table \thetable\ continued \\
\endhead
\multicolumn{1}{r}{\small Continued on the next page} \\
\endfoot
\endlastfoot
\begin{tabular}{@{}p{76pt}*{6}{>{\centering\arraybackslash}p{\dimexpr(\linewidth-76pt-12\tabcolsep)/6\relax}}@{}}
\toprule
\multirow{2}{*}{Method} & \multicolumn{3}{c}{\textbf{Grocery}} & \multicolumn{3}{c}{\textbf{Reddit-S}} \\
\cmidrule(lr){2-4}\cmidrule(lr){5-7}
 & mAP@10 & P@10 & Sem. R@1 & mAP@10 & P@10 & Sem. R@1 \\
\midrule
\multicolumn{7}{l}{\textbf{Vision-language learning}} \\
VSE++ & \tabstat{39.67}{0.13} & \tabstat{47.82}{0.14} & \tabstat{63.75}{0.37} & \tabstat{56.19}{0.39} & \tabstat{62.99}{0.27} & \tabstat{76.16}{0.46} \\
CLIP-Refine & \tabstat{34.81}{0.32} & \tabstat{42.92}{0.37} & \tabstat{60.87}{0.57} & \tabstat{59.37}{2.67} & \tabstat{65.76}{2.79} & \tabstat{75.77}{0.35} \\
SigLIP & \tabsecondstat{44.17}{0.08} & \tabsecondstat{52.06}{0.11} & \tabsecondstat{66.15}{0.28} & \tabstat{62.95}{0.11} & \tabstat{69.05}{0.14} & \tabstat{77.85}{0.34} \\
\midrule
\multicolumn{7}{l}{\textbf{Graph encoders}} \\
GCN & \tabstat{36.50}{1.06} & \tabstat{43.39}{0.97} & \tabstat{54.02}{0.92} & \tabstat{64.61}{0.86} & \tabstat{69.61}{0.77} & \tabstat{77.65}{0.44} \\
GAT & \tabstat{35.67}{0.64} & \tabstat{42.49}{0.47} & \tabstat{52.04}{1.38} & \tabstat{67.93}{0.67} & \tabstat{72.45}{0.36} & \tabstat{80.46}{0.80} \\
GraphSAGE & \tabstat{28.95}{1.89} & \tabstat{35.41}{2.25} & \tabstat{49.02}{1.46} & \tabstat{58.01}{1.22} & \tabstat{63.20}{0.93} & \tabstat{74.59}{1.12} \\
\midrule
\multicolumn{7}{l}{\textbf{Multimodal graph learning}} \\
MMGCN & \tabstat{35.21}{0.82} & \tabstat{42.00}{0.94} & \tabstat{52.56}{0.82} & \tabstat{64.74}{1.29} & \tabstat{69.92}{1.07} & \tabstat{76.55}{1.09} \\
MGAT & \tabstat{40.65}{0.75} & \tabstat{48.06}{0.82} & \tabstat{61.40}{0.25} & \tabstat{67.43}{1.69} & \tabstat{72.77}{1.39} & \tabstat{78.66}{0.72} \\
LGMRec & \tabstat{41.80}{1.64} & \tabstat{48.86}{1.79} & \tabstat{63.69}{1.50} & \tabsecondstat{69.29}{0.44} & \tabsecondstat{73.61}{0.37} & \tabsecondstat{81.66}{0.88} \\
DMGC & \tabstat{32.56}{11.45} & \tabstat{40.21}{11.74} & \tabstat{58.96}{9.09} & \tabstat{55.13}{1.47} & \tabstat{62.08}{1.31} & \tabstat{73.95}{0.32} \\
DGF & \tabstat{34.57}{0.54} & \tabstat{41.73}{0.58} & \tabstat{60.45}{1.57} & \tabstat{61.93}{0.14} & \tabstat{67.66}{0.13} & \tabstat{78.49}{0.64} \\
CAMPA & \tabstat{30.12}{1.15} & \tabstat{38.17}{1.14} & \tabstat{55.33}{0.96} & \tabstat{56.32}{0.59} & \tabstat{62.63}{0.30} & \tabstat{74.17}{0.32} \\
CoMAG & \tabstat{40.31}{1.91} & \tabstat{47.43}{1.92} & \tabstat{62.88}{0.32} & \tabstat{67.28}{1.53} & \tabstat{71.72}{1.59} & \tabstat{81.17}{1.30} \\
RoleMAG & \tabstat{38.11}{0.32} & \tabstat{44.73}{0.30} & \tabstat{57.43}{0.77} & \tabstat{65.43}{1.21} & \tabstat{70.36}{1.07} & \tabstat{78.79}{0.41} \\
\midrule
\multicolumn{7}{l}{\textbf{Proposed method}} \\
\textbf{GOMA} & \tabbeststat{49.60}{0.05} & \tabbeststat{56.13}{0.03} & \tabbeststat{67.82}{0.15} & \tabbeststat{78.97}{0.05} & \tabbeststat{81.80}{0.03} & \tabbeststat{85.26}{0.14} \\
\bottomrule
\end{tabular}\\
\begin{tabular}{@{}p{76pt}*{6}{>{\centering\arraybackslash}p{\dimexpr(\linewidth-76pt-12\tabcolsep)/6\relax}}@{}}
\toprule
\multirow{2}{*}{Method} & \multicolumn{3}{c}{\textbf{Toys}} & \multicolumn{3}{c}{\textbf{Books}} \\
\cmidrule(lr){2-4}\cmidrule(lr){5-7}
 & mAP@10 & P@10 & Sem. R@1 & mAP@10 & P@10 & Sem. R@1 \\
\midrule
\multicolumn{7}{l}{\textbf{Vision-language learning}} \\
VSE++ & \tabstat{29.58}{0.64} & \tabstat{38.65}{0.73} & \tabstat{58.56}{0.36} & \tabstat{46.23}{0.43} & \tabstat{55.89}{0.37} & \tabstat{63.17}{0.52} \\
CLIP-Refine & \tabstat{25.79}{0.76} & \tabstat{34.62}{0.74} & \tabstat{55.76}{0.45} & \tabstat{49.98}{0.02} & \tabstat{59.35}{0.02} & \tabstat{65.60}{0.08} \\
SigLIP & \tabsecondstat{38.44}{0.10} & \tabsecondstat{47.70}{0.11} & \tabstat{62.79}{0.08} & \tabstat{48.84}{0.13} & \tabstat{58.35}{0.15} & \tabstat{64.73}{0.04} \\
\midrule
\multicolumn{7}{l}{\textbf{Graph encoders}} \\
GCN & \tabstat{36.15}{0.93} & \tabstat{42.85}{0.95} & \tabstat{60.64}{0.85} & \tabstat{62.90}{0.21} & \tabstat{69.83}{0.19} & \tabstat{75.16}{0.21} \\
GAT & \tabstat{37.76}{0.99} & \tabstat{44.65}{1.28} & \tabstat{61.30}{0.26} & \tabstat{62.03}{0.23} & \tabstat{69.00}{0.19} & \tabstat{74.21}{0.41} \\
GraphSAGE & \tabstat{25.90}{1.64} & \tabstat{32.37}{1.92} & \tabstat{54.12}{0.64} & \tabstat{61.06}{0.40} & \tabstat{68.38}{0.32} & \tabstat{74.26}{0.31} \\
\midrule
\multicolumn{7}{l}{\textbf{Multimodal graph learning}} \\
MMGCN & \tabstat{26.84}{0.65} & \tabstat{34.19}{0.72} & \tabstat{53.17}{0.93} & \tabstat{61.14}{0.57} & \tabstat{68.40}{0.47} & \tabstat{73.65}{0.30} \\
MGAT & \tabstat{29.65}{0.04} & \tabstat{37.69}{0.09} & \tabstat{56.88}{0.37} & \tabstat{60.58}{0.43} & \tabstat{68.05}{0.41} & \tabstat{73.44}{0.07} \\
LGMRec & \tabstat{36.89}{0.35} & \tabstat{44.05}{0.45} & \tabsecondstat{62.94}{0.92} & \tabsecondstat{62.99}{0.40} & \tabsecondstat{70.32}{0.35} & \tabsecondstat{75.68}{0.33} \\
DMGC & \tabstat{13.53}{0.58} & \tabstat{20.38}{0.73} & \tabstat{43.00}{1.42} & \tabstat{48.56}{0.23} & \tabstat{57.91}{0.23} & \tabstat{64.36}{0.04} \\
DGF & \tabstat{27.34}{3.98} & \tabstat{35.01}{4.08} & \tabstat{56.55}{0.64} & \tabstat{59.93}{0.32} & \tabstat{67.85}{0.32} & \tabstat{72.55}{0.15} \\
CAMPA & \tabstat{24.53}{1.37} & \tabstat{33.08}{1.34} & \tabstat{54.35}{1.48} & \tabstat{47.21}{1.07} & \tabstat{56.57}{0.72} & \tabstat{62.40}{0.68} \\
CoMAG & \tabstat{34.36}{2.69} & \tabstat{41.85}{2.58} & \tabstat{62.10}{1.26} & \tabstat{59.38}{1.28} & \tabstat{67.06}{1.12} & \tabstat{73.05}{0.89} \\
RoleMAG & \tabstat{36.35}{1.02} & \tabstat{43.13}{1.01} & \tabstat{61.43}{0.64} & \tabstat{62.26}{0.28} & \tabstat{69.43}{0.21} & \tabstat{74.93}{0.18} \\
\midrule
\multicolumn{7}{l}{\textbf{Proposed method}} \\
\textbf{GOMA} & \tabbeststat{51.31}{0.07} & \tabbeststat{58.87}{0.07} & \tabbeststat{67.72}{0.18} & \tabbeststat{66.70}{0.01} & \tabbeststat{73.19}{0.01} & \tabbeststat{78.12}{0.03} \\
\bottomrule
\end{tabular}\\
\end{longtable}
\endgroup
 
Table~\ref{tab:complete_relation} reports relation MRR and Hits@1/10. A candidate is correct when it is the held-out linked object, and the target relationship is absent from the permitted observed graph. GOMA's leading means extend from reciprocal rank to both hit rates, supporting the use of the learned semantic initialization and observed neighborhoods for missing-relation ranking.
\begingroup\completeTableSetup\setlength{\LTcapwidth}{\linewidth}\setlength{\LTleft}{0pt}\setlength{\LTright}{0pt}
\begin{longtable}{@{}p{\linewidth}@{}}
\caption{Complete dual-attribute relation ranking.}\label{tab:complete_relation}\\
\endfirsthead
\small Table \thetable\ continued \\
\endhead
\multicolumn{1}{r}{\small Continued on the next page} \\
\endfoot
\endlastfoot
\begin{tabular}{@{}p{76pt}*{6}{>{\centering\arraybackslash}p{\dimexpr(\linewidth-76pt-12\tabcolsep)/6\relax}}@{}}
\toprule
\multirow{2}{*}{Method} & \multicolumn{3}{c}{\textbf{Sports}} & \multicolumn{3}{c}{\textbf{Cloth}} \\
\cmidrule(lr){2-4}\cmidrule(lr){5-7}
 & MRR & Hits@1 & Hits@10 & MRR & Hits@1 & Hits@10 \\
\midrule
\multicolumn{7}{l}{\textbf{Vision-language learning}} \\
VSE++ & \tabstat{62.61}{0.77} & \tabstat{49.97}{0.94} & \tabstat{86.33}{0.38} & \tabstat{59.20}{0.87} & \tabstat{46.78}{1.00} & \tabstat{83.78}{0.56} \\
CLIP-Refine & \tabstat{67.15}{0.02} & \tabstat{55.10}{0.05} & \tabstat{89.34}{0.01} & \tabstat{62.29}{0.02} & \tabstat{49.86}{0.03} & \tabstat{86.17}{0.03} \\
SigLIP & \tabstat{67.09}{0.16} & \tabstat{54.99}{0.20} & \tabstat{89.42}{0.08} & \tabstat{61.36}{1.64} & \tabstat{49.01}{1.52} & \tabstat{85.75}{1.39} \\
\midrule
\multicolumn{7}{l}{\textbf{Graph encoders}} \\
GCN & \tabsecondstat{87.50}{0.18} & \tabsecondstat{80.68}{0.18} & \tabstat{98.11}{0.23} & \tabstat{80.34}{0.11} & \tabsecondstat{71.81}{0.15} & \tabstat{94.72}{0.15} \\
GAT & \tabstat{86.67}{0.07} & \tabstat{79.67}{0.11} & \tabstat{97.76}{0.04} & \tabstat{79.48}{0.12} & \tabstat{70.89}{0.18} & \tabstat{93.99}{0.22} \\
GraphSAGE & \tabstat{82.05}{0.42} & \tabstat{73.44}{0.76} & \tabstat{95.95}{0.45} & \tabstat{75.33}{0.03} & \tabstat{65.76}{0.02} & \tabstat{92.17}{0.39} \\
\midrule
\multicolumn{7}{l}{\textbf{Multimodal graph learning}} \\
MMGCN & \tabstat{75.48}{0.43} & \tabstat{64.35}{0.60} & \tabstat{94.72}{0.10} & \tabstat{67.29}{0.29} & \tabstat{55.53}{0.51} & \tabstat{89.23}{0.54} \\
MGAT & \tabstat{73.58}{0.41} & \tabstat{61.53}{0.65} & \tabstat{95.07}{0.33} & \tabstat{66.23}{0.45} & \tabstat{53.73}{0.55} & \tabstat{89.66}{0.33} \\
LGMRec & \tabstat{87.07}{0.11} & \tabstat{79.92}{0.24} & \tabsecondstat{98.34}{0.02} & \tabsecondstat{80.54}{0.32} & \tabstat{71.64}{0.48} & \tabsecondstat{95.57}{0.11} \\
DMGC & \tabstat{57.38}{16.31} & \tabstat{46.07}{15.04} & \tabstat{78.32}{18.33} & \tabstat{54.10}{13.77} & \tabstat{42.57}{12.70} & \tabstat{76.83}{15.65} \\
DGF & \tabstat{77.06}{0.37} & \tabstat{66.60}{0.31} & \tabstat{95.03}{0.42} & \tabstat{71.91}{0.85} & \tabstat{60.91}{1.13} & \tabstat{91.79}{0.37} \\
CAMPA & \tabstat{64.07}{0.34} & \tabstat{51.43}{0.57} & \tabstat{87.75}{0.36} & \tabstat{56.14}{2.12} & \tabstat{43.59}{2.37} & \tabstat{80.71}{1.33} \\
CoMAG & \tabstat{84.54}{0.64} & \tabstat{76.77}{0.86} & \tabstat{97.26}{0.27} & \tabstat{74.33}{2.55} & \tabstat{64.21}{3.10} & \tabstat{92.43}{1.15} \\
RoleMAG & \tabstat{85.20}{0.83} & \tabstat{77.19}{1.14} & \tabstat{97.90}{0.15} & \tabstat{76.83}{0.52} & \tabstat{66.96}{0.70} & \tabstat{93.93}{0.18} \\
\midrule
\multicolumn{7}{l}{\textbf{Proposed method}} \\
\textbf{GOMA} & \tabbeststat{87.89}{0.03} & \tabbeststat{80.98}{0.04} & \tabbeststat{98.68}{0.03} & \tabbeststat{81.32}{0.11} & \tabbeststat{72.75}{0.13} & \tabbeststat{95.70}{0.10} \\
\bottomrule
\end{tabular}\\
\end{longtable}
\endgroup
 
\subsection{Component Ablations}
\label{app:ablations}
The controls distinguish removing a computation from changing its role. The w/o graph regularization variant removes $\mathcal E_{\mathcal G}$ while retaining both mappings, both paired objectives, and semantic-guided propagation. The w/o semantic mapping variant sets $\mathbf z=\mathbf c$ and removes $g_m$. Shared readout keeps all parameters but moves identity supervision and scoring to $\mathbf z$. The w/o graph propagation variant scores the combined initialization $\mathbf H_0$ directly. These variants are retrained with their own validation scores and the same selection rule as GOMA. Active loss scales retain their full-model initialization values.

Relabeled training graph applies a vertex permutation to the training adjacency. It preserves the graph structure and degree multiset while changing their assignment to objects. The initialization controls remove either $\mathbf C$ or $\mathbf Z$ from $\mathbf H_0$, and the content-guided variant computes edge weights from $\mathbf C$. All other operations remain active. These additional controls test the interaction of learned signals with graph propagation.

Tables~\ref{tab:ablation_identity} and~\ref{tab:ablation_semantic} give the single-modality results. GOMA leads identity R@1 and semantic mAP@10 over both the same-capacity shared readout and the mapping removal on every category dataset. Graph regularization improves semantic mAP@10 over its removal on all four datasets. Relabeling the training graph also lowers this metric. These comparisons identify useful information from the observed relationships beyond the two paired objectives.

\begin{table}[htbp]
\caption{Complete bidirectional identity recalls for retrained representation, regularization, and propagation controls.}
\label{tab:ablation_identity}
\centering\completeTableSetup
\begin{tabular*}{\linewidth}{@{\extracolsep{\fill}}l*{6}{c}@{}}
\toprule
\multirow{2}{*}{Variant} & \multicolumn{3}{c}{\textbf{Grocery}} & \multicolumn{3}{c}{\textbf{Reddit-S}} \\
\cmidrule(lr){2-4}\cmidrule(lr){5-7}
 & R@1 & R@5 & R@10 & R@1 & R@5 & R@10 \\
\midrule
\textbf{GOMA} & \tabbeststat{52.96}{0.06} & \tabstat{78.57}{0.06} & \tabstat{85.14}{0.17} & \tabstat{18.66}{0.19} & \tabstat{38.13}{0.26} & \tabstat{47.37}{0.36} \\
\midrule
w/o graph regularization & \tabstat{52.79}{0.08} & \tabstat{78.43}{0.30} & \tabstat{84.93}{0.50} & \tabbeststat{18.96}{0.55} & \tabsecondstat{38.29}{0.42} & \tabstat{47.34}{0.24} \\
w/o graph propagation & \tabsecondstat{52.87}{0.09} & \tabbeststat{78.87}{0.28} & \tabsecondstat{85.41}{0.24} & \tabstat{18.70}{0.14} & \tabstat{38.17}{0.20} & \tabsecondstat{47.55}{0.09} \\
Shared readout & \tabstat{51.79}{0.13} & \tabstat{78.47}{0.22} & \tabstat{85.36}{0.39} & \tabstat{18.13}{0.30} & \tabstat{37.26}{0.51} & \tabstat{47.31}{0.35} \\
w/o semantic mapping & \tabstat{51.83}{0.12} & \tabsecondstat{78.77}{0.06} & \tabbeststat{85.50}{0.11} & \tabstat{18.21}{0.31} & \tabstat{37.41}{0.35} & \tabstat{47.46}{0.51} \\
Relabeled training graph & \tabstat{51.80}{0.15} & \tabstat{78.22}{0.13} & \tabstat{85.07}{0.09} & \tabstat{18.47}{0.09} & \tabstat{37.96}{0.54} & \tabstat{46.80}{0.70} \\
w/o content initialization & \tabbeststat{52.96}{0.06} & \tabstat{78.57}{0.06} & \tabstat{85.14}{0.17} & \tabstat{18.62}{0.24} & \tabstat{38.10}{0.32} & \tabstat{47.39}{0.33} \\
w/o semantic initialization & \tabsecondstat{52.87}{0.09} & \tabbeststat{78.87}{0.28} & \tabsecondstat{85.41}{0.24} & \tabsecondstat{18.76}{0.12} & \tabbeststat{38.43}{0.29} & \tabbeststat{47.72}{0.26} \\
Content-guided edges & \tabbeststat{52.96}{0.06} & \tabstat{78.57}{0.06} & \tabstat{85.14}{0.17} & \tabstat{18.70}{0.14} & \tabstat{38.17}{0.20} & \tabsecondstat{47.55}{0.09} \\
\midrule
\multirow{2}{*}{Variant} & \multicolumn{3}{c}{\textbf{Toys}} & \multicolumn{3}{c}{\textbf{Books}} \\
\cmidrule(lr){2-4}\cmidrule(lr){5-7}
 & R@1 & R@5 & R@10 & R@1 & R@5 & R@10 \\
\midrule
\textbf{GOMA} & \tabsecondstat{62.94}{0.40} & \tabstat{86.09}{0.32} & \tabstat{90.96}{0.20} & \tabstat{40.14}{0.31} & \tabstat{62.22}{0.42} & \tabstat{69.84}{0.39} \\
\midrule
w/o graph regularization & \tabstat{62.51}{0.08} & \tabstat{85.57}{0.17} & \tabstat{90.56}{0.11} & \tabbeststat{40.38}{0.12} & \tabbeststat{62.32}{0.18} & \tabbeststat{70.04}{0.16} \\
w/o graph propagation & \tabstat{62.52}{0.33} & \tabstat{85.87}{0.32} & \tabstat{90.86}{0.30} & \tabstat{40.04}{0.37} & \tabstat{62.07}{0.39} & \tabstat{69.73}{0.39} \\
Shared readout & \tabstat{60.89}{0.23} & \tabstat{84.64}{0.04} & \tabstat{90.15}{0.13} & \tabstat{39.00}{0.10} & \tabstat{60.75}{0.06} & \tabstat{68.30}{0.08} \\
w/o semantic mapping & \tabstat{60.86}{0.17} & \tabstat{84.65}{0.23} & \tabstat{90.17}{0.21} & \tabstat{39.04}{0.07} & \tabstat{60.78}{0.02} & \tabstat{68.40}{0.08} \\
Relabeled training graph & \tabstat{62.10}{0.60} & \tabstat{85.15}{0.60} & \tabstat{90.17}{0.46} & \tabstat{40.02}{0.16} & \tabstat{62.09}{0.21} & \tabstat{69.74}{0.30} \\
w/o content initialization & \tabstat{62.87}{0.38} & \tabstat{86.07}{0.30} & \tabstat{91.00}{0.24} & \tabstat{39.96}{0.23} & \tabstat{62.02}{0.32} & \tabstat{69.70}{0.34} \\
w/o semantic initialization & \tabbeststat{62.96}{0.26} & \tabbeststat{86.20}{0.06} & \tabbeststat{91.23}{0.17} & \tabstat{39.83}{0.45} & \tabstat{61.89}{0.65} & \tabstat{69.56}{0.61} \\
Content-guided edges & \tabbeststat{62.96}{0.36} & \tabsecondstat{86.14}{0.25} & \tabsecondstat{91.09}{0.03} & \tabsecondstat{40.19}{0.36} & \tabsecondstat{62.30}{0.46} & \tabsecondstat{69.95}{0.44} \\
\bottomrule
\end{tabular*}
 \end{table}

\begin{table}[htbp]
\caption{Complete single-modality semantic scores for the same controls. R@1 uses category relevance.}
\label{tab:ablation_semantic}
\centering\completeTableSetup
\begin{tabular*}{\linewidth}{@{\extracolsep{\fill}}l*{6}{c}@{}}
\toprule
\multirow{2}{*}{Variant} & \multicolumn{3}{c}{\textbf{Grocery}} & \multicolumn{3}{c}{\textbf{Reddit-S}} \\
\cmidrule(lr){2-4}\cmidrule(lr){5-7}
 & mAP@10 & P@10 & Sem. R@1 & mAP@10 & P@10 & Sem. R@1 \\
\midrule
\textbf{GOMA} & \tabbeststat{42.82}{0.22} & \tabbeststat{50.23}{0.17} & \tabstat{60.16}{0.39} & \tabbeststat{59.74}{0.26} & \tabbeststat{65.39}{0.31} & \tabbeststat{70.13}{0.08} \\
\midrule
w/o graph regularization & \tabstat{40.69}{0.27} & \tabstat{48.64}{0.15} & \tabbeststat{61.88}{0.38} & \tabstat{52.94}{0.16} & \tabstat{59.89}{0.12} & \tabstat{68.89}{0.08} \\
w/o graph propagation & \tabsecondstat{42.41}{0.41} & \tabsecondstat{49.89}{0.33} & \tabstat{60.09}{0.17} & \tabstat{59.61}{0.17} & \tabstat{65.28}{0.21} & \tabstat{69.99}{0.19} \\
Shared readout & \tabstat{39.68}{0.16} & \tabstat{47.66}{0.11} & \tabsecondstat{60.87}{0.17} & \tabstat{52.26}{0.45} & \tabstat{59.36}{0.40} & \tabstat{67.28}{0.29} \\
w/o semantic mapping & \tabstat{39.53}{0.29} & \tabstat{47.57}{0.27} & \tabstat{60.63}{0.37} & \tabstat{52.71}{0.47} & \tabstat{59.78}{0.38} & \tabstat{68.00}{0.17} \\
Relabeled training graph & \tabstat{24.00}{1.19} & \tabstat{32.88}{1.17} & \tabstat{47.18}{1.44} & \tabstat{53.28}{0.41} & \tabstat{60.14}{0.34} & \tabstat{67.40}{0.45} \\
w/o content initialization & \tabbeststat{42.82}{0.22} & \tabbeststat{50.23}{0.17} & \tabstat{60.16}{0.39} & \tabsecondstat{59.65}{0.17} & \tabsecondstat{65.34}{0.26} & \tabsecondstat{70.06}{0.07} \\
w/o semantic initialization & \tabsecondstat{42.41}{0.41} & \tabsecondstat{49.89}{0.33} & \tabstat{60.09}{0.17} & \tabstat{59.42}{0.40} & \tabstat{65.06}{0.39} & \tabstat{70.04}{0.10} \\
Content-guided edges & \tabbeststat{42.82}{0.22} & \tabbeststat{50.23}{0.17} & \tabstat{60.16}{0.39} & \tabstat{59.61}{0.17} & \tabstat{65.28}{0.21} & \tabstat{69.99}{0.19} \\
\midrule
\multirow{2}{*}{Variant} & \multicolumn{3}{c}{\textbf{Toys}} & \multicolumn{3}{c}{\textbf{Books}} \\
\cmidrule(lr){2-4}\cmidrule(lr){5-7}
 & mAP@10 & P@10 & Sem. R@1 & mAP@10 & P@10 & Sem. R@1 \\
\midrule
\textbf{GOMA} & \tabstat{41.52}{0.44} & \tabstat{50.22}{0.48} & \tabstat{61.36}{0.30} & \tabsecondstat{46.16}{0.03} & \tabsecondstat{55.79}{0.01} & \tabstat{59.79}{0.07} \\
\midrule
w/o graph regularization & \tabstat{36.32}{0.21} & \tabstat{45.67}{0.19} & \tabstat{59.61}{0.18} & \tabstat{42.26}{0.02} & \tabstat{52.22}{0.01} & \tabstat{57.24}{0.11} \\
w/o graph propagation & \tabbeststat{41.85}{0.51} & \tabbeststat{50.60}{0.53} & \tabbeststat{61.44}{0.23} & \tabstat{46.13}{0.09} & \tabstat{55.78}{0.04} & \tabsecondstat{59.82}{0.14} \\
Shared readout & \tabstat{37.36}{0.60} & \tabstat{46.58}{0.62} & \tabstat{59.96}{0.28} & \tabstat{43.11}{0.10} & \tabstat{53.23}{0.10} & \tabstat{58.03}{0.09} \\
w/o semantic mapping & \tabstat{37.43}{0.14} & \tabstat{46.71}{0.13} & \tabstat{60.18}{0.11} & \tabstat{43.05}{0.14} & \tabstat{53.20}{0.13} & \tabstat{57.92}{0.05} \\
Relabeled training graph & \tabstat{31.06}{4.24} & \tabstat{40.25}{3.89} & \tabstat{55.52}{3.30} & \tabstat{43.81}{0.06} & \tabstat{53.56}{0.06} & \tabstat{58.10}{0.28} \\
w/o content initialization & \tabsecondstat{41.56}{0.39} & \tabsecondstat{50.29}{0.38} & \tabsecondstat{61.43}{0.25} & \tabstat{46.08}{0.04} & \tabstat{55.75}{0.03} & \tabstat{59.73}{0.10} \\
w/o semantic initialization & \tabstat{41.26}{0.16} & \tabstat{49.99}{0.16} & \tabstat{61.23}{0.10} & \tabstat{46.11}{0.08} & \tabstat{55.78}{0.01} & \tabstat{59.75}{0.14} \\
Content-guided edges & \tabstat{41.38}{0.22} & \tabstat{50.08}{0.26} & \tabstat{61.14}{0.11} & \tabbeststat{46.18}{0.05} & \tabbeststat{55.80}{0.02} & \tabbeststat{59.87}{0.07} \\
\bottomrule
\end{tabular*}
 \end{table}

Table~\ref{tab:ablation_index} reports every dual-attribute metric. Full GOMA leads both core graph removals on five datasets. On Cloth, graph-regularization removal increases MRR by 0.131 points on average. The content-initialization removal improves some Reddit-S and Toys scores, and semantic guidance is not the strongest propagation rule on the two relation datasets. These results distinguish the roles of graph regularization and graph propagation across retrieval targets. The same GOMA configuration is used throughout.

\begin{table}[htbp]
\caption{Complete dual-attribute semantic and relation scores for the retrained controls.}
\label{tab:ablation_index}
\centering\completeTableSetup
\begin{tabular*}{\linewidth}{@{\extracolsep{\fill}}l*{6}{c}@{}}
\toprule
\multirow{2}{*}{Variant} & \multicolumn{3}{c}{\textbf{Grocery}} & \multicolumn{3}{c}{\textbf{Reddit-S}} \\
\cmidrule(lr){2-4}\cmidrule(lr){5-7}
 & mAP@10 & P@10 & Sem. R@1 & mAP@10 & P@10 & Sem. R@1 \\
\midrule
\textbf{GOMA} & \tabbeststat{49.60}{0.05} & \tabsecondstat{56.13}{0.03} & \tabsecondstat{67.82}{0.15} & \tabsecondstat{78.97}{0.05} & \tabsecondstat{81.80}{0.03} & \tabstat{85.26}{0.14} \\
\midrule
w/o graph regularization & \tabsecondstat{49.55}{0.09} & \tabbeststat{56.50}{0.11} & \tabbeststat{68.15}{0.32} & \tabstat{77.24}{0.10} & \tabstat{80.40}{0.09} & \tabstat{85.72}{0.21} \\
w/o graph propagation & \tabstat{47.46}{0.25} & \tabstat{54.91}{0.19} & \tabstat{67.24}{0.42} & \tabstat{69.81}{0.18} & \tabstat{74.63}{0.13} & \tabstat{81.39}{0.42} \\
Shared readout & \tabstat{49.05}{0.05} & \tabstat{55.82}{0.07} & \tabstat{67.63}{0.20} & \tabstat{77.98}{0.16} & \tabstat{81.12}{0.11} & \tabstat{85.49}{0.28} \\
w/o semantic mapping & \tabstat{48.98}{0.06} & \tabstat{55.69}{0.13} & \tabstat{67.51}{0.04} & \tabstat{78.14}{0.07} & \tabstat{81.18}{0.06} & \tabstat{85.57}{0.09} \\
Relabeled training graph & \tabstat{46.62}{0.07} & \tabstat{53.61}{0.12} & \tabstat{66.32}{0.34} & \tabstat{77.12}{0.17} & \tabstat{80.17}{0.18} & \tabstat{85.72}{0.43} \\
w/o content initialization & \tabstat{49.34}{0.18} & \tabstat{55.90}{0.15} & \tabstat{66.84}{0.13} & \tabbeststat{79.40}{0.31} & \tabbeststat{82.16}{0.25} & \tabsecondstat{85.74}{0.37} \\
w/o semantic initialization & \tabstat{48.97}{0.08} & \tabstat{55.93}{0.07} & \tabstat{67.60}{0.06} & \tabstat{76.67}{0.35} & \tabstat{79.94}{0.30} & \tabbeststat{85.77}{0.07} \\
Content-guided edges & \tabstat{49.20}{0.08} & \tabstat{55.82}{0.04} & \tabstat{67.31}{0.14} & \tabstat{78.91}{0.08} & \tabstat{81.76}{0.06} & \tabstat{85.25}{0.11} \\
\midrule
\multirow{2}{*}{Variant} & \multicolumn{3}{c}{\textbf{Toys}} & \multicolumn{3}{c}{\textbf{Books}} \\
\cmidrule(lr){2-4}\cmidrule(lr){5-7}
 & mAP@10 & P@10 & Sem. R@1 & mAP@10 & P@10 & Sem. R@1 \\
\midrule
\textbf{GOMA} & \tabsecondstat{51.31}{0.07} & \tabsecondstat{58.87}{0.07} & \tabsecondstat{67.72}{0.18} & \tabbeststat{66.70}{0.01} & \tabbeststat{73.19}{0.01} & \tabbeststat{78.12}{0.03} \\
\midrule
w/o graph regularization & \tabstat{49.18}{0.07} & \tabstat{56.83}{0.09} & \tabstat{67.36}{0.16} & \tabsecondstat{66.62}{0.03} & \tabstat{73.02}{0.02} & \tabstat{77.99}{0.04} \\
w/o graph propagation & \tabstat{45.65}{0.41} & \tabstat{54.28}{0.44} & \tabstat{65.78}{0.36} & \tabstat{51.01}{0.14} & \tabstat{60.19}{0.11} & \tabstat{66.43}{0.14} \\
Shared readout & \tabstat{50.08}{0.25} & \tabstat{57.67}{0.26} & \tabstat{67.28}{0.30} & \tabstat{66.53}{0.02} & \tabstat{73.03}{0.02} & \tabstat{77.98}{0.08} \\
w/o semantic mapping & \tabstat{50.22}{0.11} & \tabstat{57.81}{0.12} & \tabstat{67.03}{0.23} & \tabstat{66.46}{0.04} & \tabstat{72.99}{0.03} & \tabstat{77.90}{0.06} \\
Relabeled training graph & \tabstat{47.62}{1.81} & \tabstat{55.23}{1.81} & \tabstat{66.65}{0.64} & \tabstat{66.07}{0.06} & \tabstat{72.67}{0.05} & \tabstat{77.70}{0.09} \\
w/o content initialization & \tabbeststat{51.63}{0.24} & \tabbeststat{59.08}{0.31} & \tabbeststat{68.06}{0.19} & \tabstat{66.56}{0.07} & \tabsecondstat{73.13}{0.06} & \tabbeststat{78.12}{0.13} \\
w/o semantic initialization & \tabstat{49.05}{0.14} & \tabstat{56.74}{0.15} & \tabstat{67.10}{0.16} & \tabstat{66.57}{0.06} & \tabstat{73.01}{0.07} & \tabstat{78.03}{0.08} \\
Content-guided edges & \tabstat{50.99}{0.11} & \tabstat{58.58}{0.17} & \tabstat{67.50}{0.04} & \tabsecondstat{66.62}{0.03} & \tabstat{73.12}{0.03} & \tabsecondstat{78.04}{0.05} \\
\midrule
\multirow{2}{*}{Variant} & \multicolumn{3}{c}{\textbf{Sports}} & \multicolumn{3}{c}{\textbf{Cloth}} \\
\cmidrule(lr){2-4}\cmidrule(lr){5-7}
 & MRR & Hits@1 & Hits@10 & MRR & Hits@1 & Hits@10 \\
\midrule
\textbf{GOMA} & \tabstat{87.89}{0.03} & \tabstat{80.98}{0.04} & \tabsecondstat{98.68}{0.03} & \tabstat{81.32}{0.11} & \tabstat{72.75}{0.13} & \tabstat{95.70}{0.10} \\
\midrule
w/o graph regularization & \tabstat{87.78}{0.02} & \tabstat{80.78}{0.03} & \tabstat{98.63}{0.02} & \tabbeststat{81.46}{0.00} & \tabbeststat{72.95}{0.01} & \tabsecondstat{95.77}{0.03} \\
w/o graph propagation & \tabstat{68.24}{0.11} & \tabstat{56.40}{0.13} & \tabstat{89.78}{0.09} & \tabstat{63.36}{0.15} & \tabstat{51.32}{0.13} & \tabstat{86.50}{0.17} \\
Shared readout & \tabstat{87.88}{0.04} & \tabstat{80.97}{0.06} & \tabstat{98.66}{0.05} & \tabstat{81.34}{0.07} & \tabstat{72.75}{0.09} & \tabstat{95.74}{0.04} \\
w/o semantic mapping & \tabstat{87.89}{0.02} & \tabsecondstat{81.02}{0.04} & \tabstat{98.64}{0.04} & \tabstat{81.38}{0.07} & \tabstat{72.81}{0.11} & \tabstat{95.76}{0.08} \\
Relabeled training graph & \tabbeststat{88.00}{0.04} & \tabbeststat{81.12}{0.02} & \tabbeststat{98.73}{0.05} & \tabstat{81.36}{0.01} & \tabstat{72.76}{0.02} & \tabstat{95.76}{0.07} \\
w/o content initialization & \tabstat{87.88}{0.03} & \tabstat{80.94}{0.02} & \tabsecondstat{98.68}{0.04} & \tabsecondstat{81.40}{0.02} & \tabsecondstat{72.83}{0.04} & \tabbeststat{95.79}{0.01} \\
w/o semantic initialization & \tabsecondstat{87.91}{0.02} & \tabstat{81.00}{0.04} & \tabstat{98.67}{0.03} & \tabstat{81.37}{0.04} & \tabstat{72.81}{0.06} & \tabstat{95.75}{0.01} \\
Content-guided edges & \tabsecondstat{87.91}{0.03} & \tabsecondstat{81.02}{0.07} & \tabstat{98.67}{0.03} & \tabstat{81.39}{0.06} & \tabstat{72.81}{0.08} & \tabsecondstat{95.77}{0.03} \\
\bottomrule
\end{tabular*}
 \end{table}

\begin{figure}[htbp]
\centering
\includegraphics[width=\linewidth]{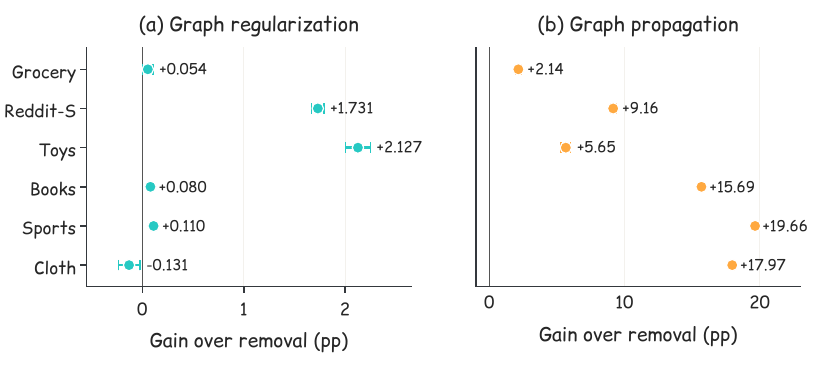}
\caption{Paired GOMA gains over removal of graph regularization (a) and graph propagation (b). Points show means over seeds 43--45, and horizontal bars show the sample SD of paired seed differences. The first four datasets use semantic mAP@10 and the last two use relation MRR. Gains are in percentage points.}
\label{fig:paired_core}
\end{figure}

Fig.~\ref{fig:paired_core} gives paired uncertainty for the core ablations. GOMA's primary gains over graph-regularization removal and graph-propagation removal are positive in every seed on Grocery, Reddit-S, Toys, Books, and Sports. The negative graph-regularization difference on Cloth also occurs in every seed.

\subsection{Propagation Interventions}
\label{app:index_controls}
Table~\ref{tab:index_controls} retains the selected GOMA parameters and changes only the indicated propagation computation. Removing propagation, removing restart, or relabeling the observed graph tests the use of observed relationships and the initial signal. Uniform edge weights retain the graph and initialization. Content guidance substitutes $\mathbf C$ for the semantic guide, and permuted semantic guidance changes only the guide-to-object assignment. The scale is estimated from the legal training edges before the diagnostic permutation. Initialization removals retain the full semantic-guided operator.

\begin{table}[htbp]
\caption{Complete propagation interventions at the same GOMA weights. Only the named computation changes.}
\label{tab:index_controls}
\centering\completeTableSetup
\begin{tabular*}{\linewidth}{@{\extracolsep{\fill}}l*{6}{c}@{}}
\toprule
\multirow{2}{*}{Propagation setting} & \multicolumn{3}{c}{\textbf{Grocery}} & \multicolumn{3}{c}{\textbf{Reddit-S}} \\
\cmidrule(lr){2-4}\cmidrule(lr){5-7}
 & mAP@10 & P@10 & Sem. R@1 & mAP@10 & P@10 & Sem. R@1 \\
\midrule
GOMA & \tabstat{49.60}{0.05} & \tabstat{56.13}{0.03} & \tabstat{67.82}{0.15} & \tabstat{78.97}{0.05} & \tabstat{81.80}{0.03} & \tabstat{85.26}{0.14} \\
w/o graph propagation & \tabstat{47.59}{0.08} & \tabstat{54.96}{0.07} & \tabstat{67.63}{0.92} & \tabstat{69.83}{0.16} & \tabstat{74.66}{0.12} & \tabstat{81.23}{0.44} \\
w/o restart & \tabstat{43.31}{0.13} & \tabstat{50.13}{0.11} & \tabstat{58.65}{0.19} & \tabstat{77.59}{0.06} & \tabstat{80.68}{0.04} & \tabstat{84.35}{0.21} \\
Relabeled observed graph & \tabstat{36.32}{0.61} & \tabstat{45.04}{0.63} & \tabstat{56.18}{0.54} & \tabstat{29.08}{0.64} & \tabstat{40.79}{0.55} & \tabstat{40.66}{0.55} \\
Uniform edge weights & \tabstat{48.85}{0.09} & \tabstat{55.50}{0.02} & \tabstat{66.86}{0.23} & \tabstat{78.95}{0.08} & \tabstat{81.79}{0.06} & \tabstat{85.36}{0.13} \\
Content-guided edges & \tabstat{49.20}{0.08} & \tabstat{55.82}{0.04} & \tabstat{67.31}{0.14} & \tabstat{78.93}{0.07} & \tabstat{81.79}{0.05} & \tabstat{85.26}{0.11} \\
Permuted semantic guide & \tabstat{48.89}{0.11} & \tabstat{55.59}{0.08} & \tabstat{66.78}{0.11} & \tabstat{78.95}{0.06} & \tabstat{81.81}{0.04} & \tabstat{85.31}{0.19} \\
w/o content initialization & \tabstat{49.34}{0.18} & \tabstat{55.90}{0.15} & \tabstat{66.84}{0.13} & \tabstat{79.28}{0.13} & \tabstat{82.06}{0.10} & \tabstat{85.70}{0.31} \\
w/o semantic initialization & \tabstat{48.92}{0.08} & \tabstat{55.94}{0.07} & \tabstat{67.31}{0.24} & \tabstat{76.44}{0.28} & \tabstat{79.77}{0.25} & \tabstat{85.64}{0.19} \\
\midrule
\multirow{2}{*}{Propagation setting} & \multicolumn{3}{c}{\textbf{Toys}} & \multicolumn{3}{c}{\textbf{Books}} \\
\cmidrule(lr){2-4}\cmidrule(lr){5-7}
 & mAP@10 & P@10 & Sem. R@1 & mAP@10 & P@10 & Sem. R@1 \\
\midrule
GOMA & \tabstat{51.31}{0.07} & \tabstat{58.87}{0.07} & \tabstat{67.72}{0.18} & \tabstat{66.70}{0.01} & \tabstat{73.19}{0.01} & \tabstat{78.12}{0.03} \\
w/o graph propagation & \tabstat{45.43}{0.21} & \tabstat{54.10}{0.20} & \tabstat{65.58}{0.26} & \tabstat{51.09}{0.13} & \tabstat{60.26}{0.08} & \tabstat{66.48}{0.03} \\
w/o restart & \tabstat{49.50}{0.11} & \tabstat{57.10}{0.09} & \tabstat{65.04}{0.11} & \tabstat{65.14}{0.03} & \tabstat{71.63}{0.03} & \tabstat{76.44}{0.05} \\
Relabeled observed graph & \tabstat{31.24}{0.38} & \tabstat{41.67}{0.45} & \tabstat{50.27}{1.10} & \tabstat{47.13}{0.29} & \tabstat{56.84}{0.18} & \tabstat{63.39}{0.36} \\
Uniform edge weights & \tabstat{50.95}{0.10} & \tabstat{58.50}{0.10} & \tabstat{67.18}{0.37} & \tabstat{66.52}{0.03} & \tabstat{73.05}{0.04} & \tabstat{77.97}{0.09} \\
Content-guided edges & \tabstat{51.07}{0.11} & \tabstat{58.63}{0.12} & \tabstat{67.62}{0.23} & \tabstat{66.63}{0.02} & \tabstat{73.13}{0.02} & \tabstat{77.99}{0.04} \\
Permuted semantic guide & \tabstat{50.90}{0.09} & \tabstat{58.48}{0.06} & \tabstat{67.14}{0.22} & \tabstat{66.53}{0.03} & \tabstat{73.06}{0.02} & \tabstat{77.92}{0.06} \\
w/o content initialization & \tabstat{51.64}{0.21} & \tabstat{59.10}{0.26} & \tabstat{68.11}{0.10} & \tabstat{66.56}{0.05} & \tabstat{73.12}{0.04} & \tabstat{78.08}{0.12} \\
w/o semantic initialization & \tabstat{48.92}{0.22} & \tabstat{56.64}{0.19} & \tabstat{67.19}{0.14} & \tabstat{66.55}{0.06} & \tabstat{72.98}{0.06} & \tabstat{78.10}{0.05} \\
\midrule
\multirow{2}{*}{Propagation setting} & \multicolumn{3}{c}{\textbf{Sports}} & \multicolumn{3}{c}{\textbf{Cloth}} \\
\cmidrule(lr){2-4}\cmidrule(lr){5-7}
 & MRR & Hits@1 & Hits@10 & MRR & Hits@1 & Hits@10 \\
\midrule
GOMA & \tabstat{87.89}{0.03} & \tabstat{80.98}{0.04} & \tabstat{98.68}{0.03} & \tabstat{81.32}{0.11} & \tabstat{72.75}{0.13} & \tabstat{95.70}{0.10} \\
w/o graph propagation & \tabstat{67.93}{0.11} & \tabstat{55.99}{0.13} & \tabstat{89.66}{0.04} & \tabstat{63.36}{0.13} & \tabstat{51.12}{0.32} & \tabstat{86.84}{0.16} \\
w/o restart & \tabstat{88.32}{0.01} & \tabstat{81.63}{0.00} & \tabstat{98.80}{0.05} & \tabstat{81.27}{0.10} & \tabstat{72.98}{0.12} & \tabstat{95.42}{0.09} \\
Relabeled observed graph & \tabstat{50.63}{2.93} & \tabstat{38.24}{2.73} & \tabstat{74.45}{3.65} & \tabstat{48.37}{2.70} & \tabstat{36.69}{2.25} & \tabstat{71.52}{3.44} \\
Uniform edge weights & \tabstat{88.40}{0.02} & \tabstat{81.59}{0.04} & \tabstat{98.91}{0.02} & \tabstat{81.96}{0.10} & \tabstat{73.48}{0.15} & \tabstat{96.06}{0.06} \\
Content-guided edges & \tabstat{87.91}{0.03} & \tabstat{81.00}{0.05} & \tabstat{98.69}{0.03} & \tabstat{81.35}{0.09} & \tabstat{72.77}{0.10} & \tabstat{95.73}{0.07} \\
Permuted semantic guide & \tabstat{88.37}{0.02} & \tabstat{81.54}{0.04} & \tabstat{98.89}{0.03} & \tabstat{81.89}{0.09} & \tabstat{73.35}{0.12} & \tabstat{96.05}{0.07} \\
w/o content initialization & \tabstat{87.89}{0.04} & \tabstat{81.00}{0.10} & \tabstat{98.68}{0.04} & \tabstat{81.28}{0.14} & \tabstat{72.71}{0.14} & \tabstat{95.65}{0.14} \\
w/o semantic initialization & \tabstat{87.91}{0.03} & \tabstat{80.99}{0.04} & \tabstat{98.70}{0.01} & \tabstat{81.37}{0.08} & \tabstat{72.81}{0.11} & \tabstat{95.74}{0.07} \\
\bottomrule
\end{tabular*}
 \end{table}

The observed graph improves all six primary measures over graph removal and graph relabeling. Semantic guidance improves mAP@10 over uniform weights, content guidance, and permuted semantic guidance on every category dataset, with positive differences in each seed. On Sports and Cloth, uniform and content-guided weights give higher MRR. Thus the additional benefit of the learned semantic guide is supported most directly by category retrieval. GOMA's relation results remain above the strongest external comparisons.

Removing restart reduces the primary measure on five datasets and slightly improves Sports MRR. Removing the content or semantic part of the initialization also gives mixed changes. The full initialization and restart preserve access to both learned outputs, while their ranking effect depends on the retrieval target and graph. The tables retain these differences alongside the signal measurements.

\subsection{Empirical Probes}
\label{app:empirical}
For graph-category agreement, we remove self edges and collapse repeated undirected training edges. The observed same-category proportion is deterministic. Three relabelings use seeds 1001, 1002, and 1003 and permute vertices within exact training-degree groups. Each object retains its degree in this descriptive control. The resulting SD measures graph-relabeling variation, separately from training-seed variation.

The shared-output probe uses two residual mappings per modality. Identity ranking, paired cross-entropy, and an unnormalized edge-average graph regularizer supervise its final representation. Intermediate and final outputs are normalized and scored using the same validation candidates at the jointly selected model. This is a representation probe of shared supervision, separate from the final GOMA comparison. Table~\ref{tab:empirical_complete} reports final-minus-intermediate changes and the current GOMA embedding measurements. Table~\ref{tab:readout} tests the distinct question of whether supervision at separate outputs helps when parameter counts are equal.

\begin{table}[htbp]
\caption{Graph-category agreement, shared-output probe changes, and GOMA validation embeddings. Graph agreement uses three relabelings and retrieval uses three model seeds.}
\label{tab:empirical_complete}
\centering\completeTableSetup
\begin{tabular*}{\linewidth}{@{\extracolsep{\fill}}l*{4}{c}@{}}
\toprule
Quantity & \textbf{Grocery} & \textbf{Reddit-S} & \textbf{Toys} & \textbf{Books} \\
\midrule
Observed same-category edges (\%) & 68.21 & 96.32 & 72.68 & 67.78 \\
Relabeled same-category edges (\%) & \tabstat{9.66}{0.09} & \tabstat{73.65}{0.05} & \tabstat{9.71}{0.18} & \tabstat{16.20}{0.01} \\
\midrule
Probe final minus intermediate R@1 (Validation) & \tabstat{-0.99}{0.17} & \tabstat{-1.10}{0.30} & \tabstat{-1.93}{0.17} & \tabstat{0.84}{0.59} \\
Probe final minus intermediate mAP@10 & \tabstat{1.18}{0.17} & \tabstat{1.42}{0.22} & \tabstat{2.73}{0.07} & \tabstat{0.82}{0.11} \\
\midrule
Content embedding identity R@1 & \tabstat{52.79}{0.19} & \tabstat{18.62}{0.02} & \tabstat{63.57}{0.26} & \tabstat{50.65}{0.57} \\
Content embedding semantic mAP@10 & \tabstat{33.12}{0.42} & \tabstat{39.55}{1.30} & \tabstat{30.70}{0.80} & \tabstat{37.28}{0.39} \\
Semantic embedding identity R@1 & \tabstat{39.79}{0.18} & \tabstat{13.07}{0.19} & \tabstat{48.55}{0.94} & \tabstat{38.05}{1.89} \\
Semantic embedding semantic mAP@10 & \tabstat{43.48}{0.36} & \tabstat{58.63}{0.20} & \tabstat{41.41}{0.41} & \tabstat{43.33}{0.14} \\
\bottomrule
\end{tabular*}
 \end{table}

The propagation probe compares graph regularization with its removal at a common training step per dataset and seed. Both variants initialize from $\mathcal N(\mathbf C+\mathbf Z)$ and use the same eight-step restart rule with uniform observed-edge weights. The common step is the earlier of their selected training steps, as specified before this probe. We evaluate the initial and propagated representations using the same candidates. Table~\ref{tab:empirical_interaction} gives the graph-regularization gain before and after propagation and their difference. Every initial gain is positive, every interaction is negative, and propagation improves both variants in all seeds. This identifies overlapping benefits without assigning a unique mechanism to the negative interaction.

\begin{table}[htbp]
\caption{Graph-regularization gains before and after uniform graph propagation at matched training steps. The first four columns use mAP@10 and the last two use MRR. Entries report paired means and sample SDs in percentage points.}
\label{tab:empirical_interaction}
\centering\completeTableSetup
\begin{tabular*}{\linewidth}{@{\extracolsep{\fill}}l*{6}{c}@{}}
\toprule
Graph-regularization gain & \textbf{Grocery} & \textbf{Reddit-S} & \textbf{Toys} & \textbf{Books} & \textbf{Sports} & \textbf{Cloth} \\
\midrule
Before propagation & \tabstat{2.61}{0.28} & \tabstat{5.29}{0.35} & \tabstat{6.37}{0.20} & \tabstat{2.52}{0.12} & \tabstat{0.48}{0.08} & \tabstat{0.31}{0.11} \\
After uniform propagation & \tabstat{-0.39}{0.02} & \tabstat{1.44}{0.23} & \tabstat{1.69}{0.08} & \tabstat{-0.11}{0.04} & \tabstat{0.09}{0.01} & \tabstat{-0.06}{0.05} \\
After minus before & \tabstat{-3.01}{0.28} & \tabstat{-3.85}{0.20} & \tabstat{-4.68}{0.26} & \tabstat{-2.63}{0.15} & \tabstat{-0.39}{0.09} & \tabstat{-0.36}{0.16} \\
\bottomrule
\end{tabular*}
 \end{table}

Fig.~\ref{fig:signals}(a) averages raw squared endpoint differences over feature coordinates and modalities, then computes each seed's semantic-embedding reduction relative to its content embedding. This diagnostic differs from the degree-normalized training regularizer. Identity curves first average the two directional empirical rank distributions within each dataset, then give the four category datasets equal weight. Their mean and band summarize three seed curves. The edge-weight curves in Fig.~\ref{fig:index}(a) average empirical cumulative distributions over seeds within each dataset. Both energy measurements in Fig.~\ref{fig:index}(c) use that model's semantic-guided operator.

\section{Mathematical Analysis}
\label{app:theory}
\subsection{Normalized Graph Dirichlet Energy}
Let $\mathbf A_{\mathcal T}$ be the binary symmetric training adjacency and $\mathbf D_{\mathcal T}$ its degree matrix. On non-isolated training nodes, define the normalized Laplacian $\mathbf L_{\mathcal T}=\mathbf I-\mathbf D_{\mathcal T}^{-1/2}\mathbf A_{\mathcal T}\mathbf D_{\mathcal T}^{-1/2}$. Isolated nodes contribute zero to the edge energy. The normalized Dirichlet identity is
\begin{equation}
\operatorname{tr}(\mathbf Z^\top\mathbf L_{\mathcal T}\mathbf Z)
=\sum_{\{i,j\}\in\mathcal E_{\mathcal T}}
\left\|\frac{\mathbf z_i}{\sqrt{d_i}}-\frac{\mathbf z_j}{\sqrt{d_j}}\right\|_2^2.
\label{eq:dirichlet_identity}
\end{equation}
The sum here contains each undirected edge once. Expanding it gives degree-normalized squared norms and adjacency cross terms. Equation~\ref{eq:energy} averages the two modality energies over feature coordinates and edges. A bidirected edge list repeats each term twice in both the numerator and denominator, giving the same average. Degrees come from the complete training-induced graph, independently of the sampled edge batch.

Under the unit-norm representation used by GOMA, each edge term can also be written as
\begin{equation}
\left\|\frac{\mathbf z_i}{\sqrt{d_i}}-\frac{\mathbf z_j}{\sqrt{d_j}}\right\|_2^2
=\frac{1}{d_i}+\frac{1}{d_j}-\frac{2\,\mathbf z_i^\top\mathbf z_j}{\sqrt{d_i d_j}}.
\label{eq:energy_cosine}
\end{equation}
Thus the energy rewards degree-weighted cosine alignment together with degree-dependent norm terms. Zero energy means that $\mathbf z_i/\sqrt{d_i}$ is constant on each connected component. On an irregular degree graph, this is not the same as a simple constant embedding over the whole graph, and unit-norm embeddings cannot attain it across an edge whose endpoint degrees differ.

For $\mathbf L_{\mathcal T}=\mathbf U\boldsymbol\Lambda\mathbf U^\top$, this energy equals $\sum_\ell\lambda_\ell\|\mathbf u_\ell^\top\mathbf Z\|_2^2$. It penalizes variation associated with larger eigenvalues of the training Laplacian. The content embedding is a jointly learned representation with direct identity supervision. It is not constructed as a complementary frequency band. Both its mapping and the semantic mapping receive gradients from the semantic objective.

\subsection{Restart Propagation}
Write $\beta=1-\alpha$. Expanding Equation~\ref{eq:propagation} by induction gives
\begin{equation}
\mathbf H_K=
\left[\alpha\sum_{j=0}^{K-1}\beta^j\mathbf P_W^j+\beta^K\mathbf P_W^K\right]\mathbf H_0.
\label{eq:finite_filter}
\end{equation}
For $K=1$, the expression is $\alpha\mathbf H_0+\beta\mathbf P_W\mathbf H_0$, which is the first update. Substituting the expression for $K$ into the next update adds the $j=0$ term and shifts each existing power by one. This establishes the formula for all positive integer depths. The edge weights have already been computed from the learned semantic signal and remain unchanged during these iterations. The result is a polynomial in that operator, with no eigendecomposition required during execution.

Because semantic edge weights are positive and symmetric, the weighted normalized adjacency on non-isolated components has spectrum in $[-1,1]$. Adding the identity and dividing by two places the spectrum of $\mathbf P_W$ in $[0,1]$. Isolated nodes use eigenvalue one by construction. For $0<\alpha\leq1$, the polynomial $p_K$ has nonnegative coefficients, is nondecreasing on $[0,1]$, satisfies $p_K(1)=1$, and lies between $\alpha$ and one for $K\geq1$. Hence the finite filter emphasizes components with smaller eigenvalues of $\mathbf I-\mathbf P_W$, while maintaining a positive contribution from the initial signal in every frequency component.

As depth increases, the geometric series converges because $\beta<1$. The limit is
\begin{equation}
\mathbf H_\infty=\alpha(\mathbf I-\beta\mathbf P_W)^{-1}\mathbf H_0.
\label{eq:restart_limit}
\end{equation}
It is also the unique minimizer of
\begin{equation}
\alpha\|\mathbf H-\mathbf H_0\|_F^2+
\beta\operatorname{tr}\!\left(\mathbf H^\top(\mathbf I-\mathbf P_W)\mathbf H\right).
\label{eq:restart_objective}
\end{equation}
Differentiation yields $(\alpha\mathbf I+\beta(\mathbf I-\mathbf P_W))\mathbf H=\alpha\mathbf H_0$, which is equivalent to Equation~\ref{eq:restart_limit}. Positive $\alpha$ makes the quadratic strictly convex. GOMA uses the eight-step finite filter in the reported experiments. These identities explain restart propagation for a given weighted operator \citep{APPNP2019}. The operator itself depends nonlinearly on the learned semantic signal, so this conditional quadratic does not describe the entire learned model.

\subsection{Normalization and Graph Operators}
The final propagated representation is $\widehat{\mathbf H}=\mathcal N(\mathbf H_K)$. Row normalization scales nodes by factors determined by their filtered norms. It generally does not commute with $\mathbf P_W$, so the final representation is not described solely by the polynomial response in Equation~\ref{eq:finite_filter}. Fig.~\ref{fig:index}(c) reports the measured energy on both sides of this operation. With the relative energy defined by $\operatorname{tr}(\mathbf H^\top(\mathbf I-\mathbf P_W)\mathbf H)/\|\mathbf H\|_F^2$, changes in the figure are properties of the actual signals used in the computation.

Because every row of $\widehat{\mathbf H}$ has unit norm, its final inner product has a comparable cosine-style scale across objects. Row normalization is an independent nonlinear step applied after propagation. The low-pass polynomial in Equation~\ref{eq:finite_filter} describes only the unnormalized signal $\mathbf H_K$, so it cannot directly describe the final normalized representation.

Graph regularization and graph propagation use different forms of graph smoothness. Training minimizes an edge-average energy through nonlinear, normalized semantic mappings on training objects. Propagation applies a semantic-guided lazy symmetric operator to the combined content and semantic initialization over the permitted observed graph. The first transfers relational information into independently computable attribute representations. The second combines these representations with observed neighborhoods when that information is available. The learned semantic signal connects them by contributing to both initialization and edge weights. The training-induced and observed retrieval graphs can contain different nodes and edges, and propagation uses weighted degrees.

\section{Algorithm and Implementation}
\label{app:implementation}
\begin{algorithm}[htbp]
\small
\caption{GOMA training and retrieval}
\label{alg:goma}
\begin{algorithmic}[1]
\Require Paired features, training objects $\mathcal T$, observed graph, validation pools
\Ensure One trained model and its task-appropriate retrieval embeddings
\State Form the induced training graph and the permitted observed graph
\State Initialize $f_m,g_m$ with zero final linear layers for each modality $m$
\State Set $\eta,\lambda$ from gradient norms on ten initialization batches
\For{each optimization update}
    \State Sample paired training objects and training edges independently
    \State Compute $\mathbf c^m,\mathbf z^m$ for sampled objects and edge endpoints
    \State Evaluate Equation~\ref{eq:objective} and update $f_m,g_m$ jointly with AdamW
    \If{the update is a positive multiple of five}
        \State Evaluate the joint validation score or relation MRR
        \State Save the model on a strict improvement
    \EndIf
\EndFor
\State Select one model using the dataset's validation rule
\Statex \textbf{Single-modality retrieval}
\State Encode each query and candidate using its own supplied attribute
\State Use content inner products for identity or semantic inner products for cross-object relevance
\Statex \textbf{Dual-attribute retrieval}
\State Fuse both attributes into $\mathbf C$ and $\mathbf Z$, then form $\mathbf H_0=\mathcal N(\mathbf C+\mathbf Z)$
\State Compute the semantic edge scale on legal training edges
\State Weight observed edges using Equation~\ref{eq:conductance} and form $\mathbf P_W$
\State Apply eight propagation steps using Equation~\ref{eq:propagation} and normalize the result
\State Cache the final representations in the offline retrieval index and rank by inner products
\end{algorithmic}
\end{algorithm}

\subsection{Objectives and Loss Weights}
\label{app:optimization}
GOMA directly supervises the content embedding with hard-negative identity ranking \citep{faghri2017vse++}. For a batch of $B$ paired objects, define $s^c_{ij}=(\mathbf c_i^v)^\top\mathbf c_j^t$. The identity objective is
\begin{equation}
\mathcal L_{\mathrm{id}}=
\frac{1}{B}\sum_{i=1}^{B}\left(
\left[\gamma+\max_{j\ne i}s^c_{ij}-s^c_{ii}\right]_+
+\left[\gamma+\max_{j\ne i}s^c_{ji}-s^c_{ii}\right]_+
\right),
\label{eq:identity}
\end{equation}
where $[a]_+=\max(a,0)$ and $\gamma=0.2$. Each positive is compared with its most competitive incorrect candidate in both directions. This directly supervises identity at the earlier embedding. The semantic loss described next also updates $f_m$ through $g_m$, so both embeddings share learned content and are optimized together.

Algorithm~\ref{alg:goma} summarizes the computation. Each $f_m$ and $g_m$ contains two $d$-wide linear layers with a GELU between them. Both output layers are initialized to zero, so the residual paths initially reproduce the normalized input. The weights are independent across modalities and stages. All four mappings are jointly optimized, and there is no stop-gradient operation between the content and semantic embeddings.

For semantic scores $s^z_{ij}=(\mathbf z_i^v)^\top\mathbf z_j^t$ in a batch of $B$ paired objects, the paired objective used in Equation~\ref{eq:objective} is
\begin{equation}
\mathcal L_{\mathrm{pair}}=-\frac{1}{2B}\sum_{i=1}^{B}
\left(\log\frac{\exp(s^z_{ii}/\tau)}{\sum_{j=1}^{B}\exp(s^z_{ij}/\tau)}
+\log\frac{\exp(s^z_{ii}/\tau)}{\sum_{j=1}^{B}\exp(s^z_{ji}/\tau)}\right).
\label{eq:paired_ce}
\end{equation}
The two terms give equal weight to image-to-text and text-to-image matching. Identity ranking uses Equation~\ref{eq:identity} with margin 0.2. The implementation averages the hard-negative loss over a batch while summing its two directions. The graph regularizer averages sampled degree-normalized endpoint differences over the two modalities and feature coordinates. It therefore requires no dense query-edge score matrix.

We set the loss weights from ten training-only batches at initialization. Let $a_b$ be the norm of the identity-loss gradient with respect to the content mappings, $b_b$ the norm of the paired-loss gradient with respect to the semantic mappings, and $c_b$ the norm of the graph-energy gradient with respect to the semantic mappings. If bars denote means over those ten batches, the coefficients are $\eta=\bar a/\bar b$ and $\lambda=\bar b/\bar c$. Finite nonzero means are required for a nonempty training graph. The coefficients are stored and used throughout training, including the active terms of the component controls. This sets the initial gradient scales without introducing an additional learned weighting mechanism.

\subsection{Sampling and Model Selection}
Each update samples up to 1,024 paired training objects without replacement and 1,024 training edges with replacement. The union of these object IDs is encoded once per modality. The training edge list is restricted to pairs of training endpoints before sampling. AdamW uses learning rate 0.001 and weight decay 0.0001 \citep{adamw}. The maximum trajectory length is 500 updates, and validation is performed every five updates. The first eligible selected model is at update five, and equal validation values retain the earlier model.

For the semantic datasets, the selected score is the arithmetic mean of bidirectional identity R@1, single-modality semantic mAP@10, and dual-attribute semantic mAP@10. This gives the three applicable task objectives one common selection rule. Relation datasets use MRR. Patience ends a selection trajectory after 50 optimization updates without strict improvement. Additional task-specific validation states were recorded for analysis, but every GOMA number in the comparison tables uses only the common selected model. Training can continue while other tracked validation objectives remain active, within the same 500-update maximum.

\subsection{Computational Cost}
\label{app:cost}
Table~\ref{tab:cost} reports the number of trainable parameters and the observed training cost for the GOMA runs. The parameter count excludes the frozen upstream feature encoders. The update count covers the recorded training trajectory, including updates after the selected iteration. Runtime measures the recorded training and validation trajectory under the shared GPU execution environment. Peak memory is the maximum allocation reported by PyTorch during the run. These measurements describe the execution cost of the reported implementation. Fig.~\ref{fig:offline_index_cost} separately compares the stages of complete-record index construction after a selected checkpoint. It uses one warmup and three synchronized repeats for each of seeds 43--45 on the recorded NVIDIA RTX PRO 6000, excludes checkpoint and disk loading, and does not run training, validation, or report scoring. The stage and whole-index timers are measured separately, so their means are not additive.

\begin{table}[htbp]
\caption{GOMA training cost over three seeds. The shared GPU environment makes runtime dependent on concurrent workload.}
\label{tab:cost}
\centering\completeTableSetup
\begin{tabular*}{\linewidth}{@{\extracolsep{\fill}}lrrrr@{}}
\toprule
Dataset & Parameters & Updates & Training seconds & Peak allocated GiB \\
\midrule
Grocery & 2,101,248 & \tabstat{243.3}{5.8} & \tabstat{9.89}{0.14} & \tabstat{0.75}{0.02} \\
Reddit-S & 2,101,248 & \tabstat{235.0}{25.0} & \tabstat{9.34}{0.90} & \tabstat{0.72}{0.02} \\
Toys & 2,101,248 & \tabstat{271.7}{80.2} & \tabstat{12.66}{3.26} & \tabstat{0.84}{0.03} \\
Books & 132,096 & \tabstat{435.0}{58.9} & \tabstat{35.55}{4.70} & \tabstat{5.91}{0.00} \\
Sports & 132,096 & \tabstat{56.7}{2.9} & \tabstat{0.70}{0.03} & \tabstat{0.83}{0.68} \\
Cloth & 132,096 & \tabstat{60.0}{5.0} & \tabstat{0.93}{0.08} & \tabstat{1.06}{0.00} \\
\bottomrule
\end{tabular*}
 \end{table}

\begin{figure}[htbp]
\centering
\includegraphics[width=\linewidth]{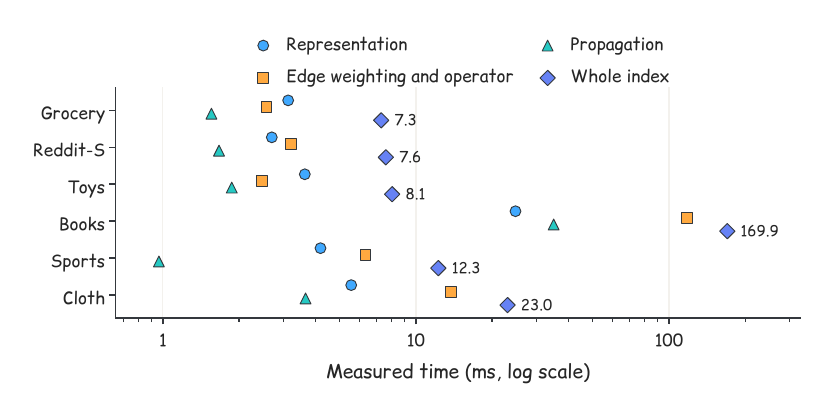}
\caption{Offline index construction time for selected GOMA checkpoints. Points show means over seeds 43--45 after one warmup and three synchronized repeats per seed. Stage and whole-index times were measured separately and need not sum. I/O and checkpoint loading are excluded.}
\label{fig:offline_index_cost}
\end{figure}

The four residual mappings contain $8d^2+8d$ trainable parameters. They require $O(nd^2)$ arithmetic when encoding $n$ objects, with a constant number of layers. A paired batch additionally forms $B^2$ similarities, while the graph regularizer uses sampled endpoint differences. The object index stores $O(Nd)$ representations and $O(N+|\mathcal E|)$ sparse graph entries. Computing edge weights takes one pass over semantic endpoint distances. Each propagation round is a sparse matrix multiplication followed by a weighted addition of the initial signal. Query scoring can be performed in chunks without changing the candidate list or relevance evaluation. The implementation uses PyTorch and NumPy, with one common random seed controlling model initialization and training samples.

\section{Data, Splits, and Evaluation}
\label{app:protocol}
\subsection{Features and Observed Graphs}
The experiments use multimodal graph data and released feature resources associated with MAGB and OpenMAG \citep{MAGB,OpenMAG}. Table~\ref{tab:dataset} summarizes the object and evaluation sizes. Grocery and Toys use item relationships from the \texttt{also\_buy} and \texttt{also\_view} metadata. Reddit-S uses the released co-comment relationship graph. These relationships are supplied independently of the category labels used for semantic evaluation. The graph operators use symmetrized binary edges with duplicate edges and self edges removed.

\begin{table}[htbp]
\caption{Dataset, observed-graph, and retrieval-pool sizes. Edge entries count both orientations. Semantic galleries have the same size as their query pools. Relation queries each rank one positive among 150 negatives.}
\label{tab:dataset}
\centering\completeTableSetup
\begin{tabular*}{\linewidth}{@{\extracolsep{\fill}}l*{6}{r}@{}}
\toprule
Quantity & \textbf{Grocery} & \textbf{Reddit-S} & \textbf{Toys} & \textbf{Books} & \textbf{Sports} & \textbf{Cloth} \\
\midrule
Objects & 17,074 & 15,894 & 20,695 & 685,294 & 50,250 & 125,839 \\
Observed edge entries & 142,262 & 283,080 & 113,402 & 13,682,358 & 603,696 & 1,622,678 \\
Feature width & 512 & 512 & 512 & 128 & 128 & 128 \\
Training objects & 10,244 & 9,536 & 12,417 & 406,689 & 30,150 & 75,503 \\
Validation queries & 1,702 & 1,584 & 2,064 & 4,000 & 4,096 & 4,096 \\
Evaluation queries & 1,702 & 1,584 & 2,064 & 10,000 & 9,825 & 9,917 \\
\bottomrule
\end{tabular*}
 \end{table}

Grocery, Reddit-S, and Toys use frozen 512-dimensional image and text projections of Qwen2-VL features \citep{qwen2vl}. Their upstream paired content models are trained without graph propagation, and all compared methods use the same projected arrays. Sports and Cloth use released 512-dimensional CLIP feature caches projected to 128 dimensions with upstream seed 44. The released cache metadata does not specify the exact encoder-weight revision. Books uses Qwen2-VL features projected to 128 dimensions by a paired content model with upstream seed 44. GOMA learns its content and semantic mappings on these common prepared features.

\subsection{Semantic Retrieval}
Grocery, Reddit-S, and Toys use common training, validation, and report identity partitions. Within each validation or report partition, the objects of each category are shuffled and divided into equal query and gallery halves. At most one object in an odd-sized category is unused, and no category is removed by a minimum-size threshold. This produces disjoint query and gallery objects while retaining all represented categories.

Books retains the released training object IDs and the ten official assigned labels $0$--$9$ provided by the release. For each class, validation uses 400 query and 400 gallery objects, while test reporting uses 1,000 query and 1,000 gallery objects. The sampling seeds are 4401 and 4402, respectively. The upstream paired projection is selected before the compared feature-adaptation methods are trained. Semantic retrieval supplies the observed graph and evaluates category relevance on the reporting objects.

Semantic retrieval averages image-to-text and text-to-image evaluation in both orientations of the disjoint query/gallery pools. For dual-attribute representations, the two modality labels yield the same score, while the two query/gallery orientations remain distinct. A candidate is relevant when it has the same category as the query and a different object identity. For a query with $R$ relevant candidates, AP@10 divides the sum of precision at relevant ranks up to ten by $\min(10,R)$. Precision at ten is the relevant fraction of the first ten candidates, and semantic R@1 records whether the highest-ranked candidate is relevant. Scores are first averaged within categories and then across categories.

\subsection{Identity Retrieval}
Identity retrieval takes the union of the corresponding semantic query and gallery IDs. Every image queries all text attributes in this union, and every text queries the image attributes. The matching object ID defines the true pair. Candidate ordering is permuted using seed 918273. IDs define relevance and candidate ordering but are not input features to GOMA. Each identity query has equal weight, and the reported primary recall averages the two modality directions. Thus an identity pool contains twice the number of objects in one semantic query half.

The strict single-modality interface also applies to candidates. An image representation is computed from image features alone, and a text representation from text features alone. In particular, the presence of a paired attribute elsewhere in the data store does not make that attribute part of the query encoder's input. Training can still use known pairs from the training partition, as in ordinary cross-modal representation learning.

\subsection{Relation Ranking}
Sports and Cloth use canonical unordered relation candidates with pool-construction seed 20260905 and a seeded random orientation for each positive. Validation contains 4,096 positive queries. Report construction begins with 10,000 positives and removes overlaps with the complete validation candidate pool, leaving 9,825 Sports queries and 9,917 Cloth queries. Target edges and reverse orientations are removed from the observed adjacency before propagation.

Each positive is ranked against 150 distinct negatives. Negative endpoints match the positive endpoint's degree bin, defined by $\lfloor\log_2(d+1)\rfloor$, expanding the bin radius when necessary. Self candidates and known positive relationships are excluded. Additional graph-support realizations are used only to exclude known positives in candidate construction and are not additional model inputs. All methods rank the same positive and negative arrays.

For positive score $s^+$ and negative scores $s_j^-$, the tie-aware rank is $r=1+\tfrac12(\sum_j\mathbb 1[s_j^->s^+]+\sum_j\mathbb 1[s_j^-\geq s^+])$. MRR averages $1/r$, while Hits@$k$ averages $\mathbb 1[r\leq k]$. A queried object and its true endpoint can retain their other observed relationships. This is the relational context available to the dual-attribute model when predicting the missing link.

\section{External Methods and Their Retrieval Adaptations}
\label{app:baselines}
\subsection{Common Inputs and Objectives}
All external numbers are obtained from runs using the same prepared features and retrieval pools as GOMA. The roster is identical for single-modality and dual-attribute retrieval. The comparison evaluates adaptations of representation, recommendation, clustering, and multimodal graph models through their output embeddings. Each method keeps its formal name, and the following details specify how its original input and output computations are used for these retrieval tasks.

VSE++ uses independent linear maps and margin-0.2 max-violation ranking. Its implementation averages the batch loss while summing the two directional terms. CLIP-Refine uses identity-initialized linear maps and its feature-level refinement objectives. The weights are zero for its additional contrastive term and one for each random-classification and knowledge-distillation term, with its $\alpha=0.5$ and $T=1$ settings. SigLIP uses independent modality maps with its pairwise sigmoid objective and learned logit scale and bias. These three methods are direct vision-language comparators for single-modality identity retrieval. GCN, GAT, GraphSAGE, and the multimodal graph methods instead serve as missing-modality adaptation controls. They retain their graph structure and auxiliary objectives while using independently computable views when one attribute is unavailable.

\subsection{Dual-Attribute Retrieval}
GCN, GAT, and GraphSAGE operate on concatenated image and text features and provide a shared node representation. MMGCN and MGAT retain modality-specific graph computation, while LGMRec retains its local and global graph construction. GCN, GAT, GraphSAGE, MMGCN, MGAT, LGMRec, CAMPA, and RoleMAG use adapted paired supervision for retrieval. RoleMAG retains its role-expert computation with the released zero auxiliary-loss coefficient. Their permitted fused or node outputs are used for dual-attribute inner-product ranking.

DMGC retains its reconstruction and clustering objectives and uses the shared representation before the classifier for scoring. The cluster count uses the known category count on semantic datasets and the inherited two-cluster setting on relation datasets. DGF retains feature filtering and graph/community objectives, with paired supervision restricted to training objects. CoMAG retains its multihop computation and shared/private representation choices under validation selection. These method-specific objectives remain active in their retrieval adaptations.

Dual-attribute external runs use a maximum of 500 updates and validation every five updates. Optimization uses Adam with nominal learning rate 0.001 and weight decay 0.0001. Pair batches can contain up to 16,384 objects. Most runs use 50-update patience, while Sports DMGC, DGF, and CoMAG use 25. Books generally uses 25, with DGF using 50. We retain each method's established training-selection protocol, while using a common maximum of 500 updates and validation every five updates. CoMAG's width-512, depth-four configuration is selected from four validation configurations on Sports. The same retained external runs supply every dual-attribute comparison in this paper.

\subsection{Single-Modality Retrieval}
Single-modality encoding supplies only the visible attribute of one query or candidate object. Training retains each graph method's layers, modality branches, aggregation or filtering operations, and auxiliary objectives. Paired retrieval supervision acts on their training graph outputs alongside each method's own auxiliary objectives. At inference, each independently encoded object uses the method's corresponding missing-input computation.

GCN, GAT, and GraphSAGE retain their graph layers and output heads, with the missing attribute set to zero in the single-object view. MMGCN and MGAT use their implementations' no-ID option. LGMRec keeps the local training graph and global hypergraph during learning, while a query uses its single-object hypergraph computation. DMGC keeps its shared low-frequency graph, modality-specific high-frequency processing, fusion, reconstruction, and clustering objectives during training. For a query, the low-frequency computation includes the object itself and the empty-neighborhood high-frequency computation corresponds to an identity Laplacian. DGF directly uses its independent modality output together with the graph and paired objectives during training.

CAMPA retains propagation and self/cross-trajectory attention. A missing trajectory starts from zero, and an unavailable neighborhood uses a zero propagation operator. CoMAG retains its full multihop matcher, shared/private representations, and orthogonality objective. Its single-modality query invokes the released model's graph-free path with the missing attribute set to zero, so query representations are computed independently. RoleMAG uses its missing-modality mask and role-based computation. These adaptations provide independent query representations without constructing a graph between the evaluation queries or accessing their missing attributes.

For the 11 graph-based adaptations, each training subgraph contains the sampled paired objects and sampled edge endpoints, together with all training edges induced by their union. Batches use 1,024 pairs and 1,024 sampled edges. Training uses Adam with learning rate 0.001, weight decay 0.0001, and at most 500 updates. DMGC and DGF retain coefficient one for their auxiliary objectives. CoMAG uses its own retrieval loss with temperature 0.1 and its configured orthogonality coefficient 0.01. Identity R@1 and semantic mAP@10 have separate validation selectors with 50-update patience and five-update evaluation intervals. The first eligible model is at update five, and ties retain the earliest model. VSE++, CLIP-Refine, and SigLIP use the same single-modality pair-batch size and trained-model selection rule. The task-specific external selections are retained in the comparison with GOMA's one jointly selected model.

\section{Limitations}
\label{app:limitations}
GOMA leads all 14 primary measures against the 14 external methods on the common prepared features. Complete directional and higher-rank recall results are reported in Table~\ref{tab:complete_identity}. All reported models are selected using the predefined validation protocol, while evaluation on additional unseen collections would further test cross-dataset generalization. Component behavior depends on the retrieval task. Semantic guidance is most consistently supported on category retrieval, while relation-ranking graphs can favor simpler edge weights. The deployment setting uses complete image-text records and a permitted observed graph to construct an offline index. Changes to the graph or learned representations require rebuilding that index. Online insertion and raw feature extraction lie outside this setting. The benchmark supplies the retrieval intent with each query rather than inferring it automatically.
 \end{document}